\documentclass[twoside]{article}

\usepackage[accepted]{aistats2025}

\usepackage[round]{natbib}

\usepackage{comment}
\usepackage{graphicx} %
\usepackage{float} %
\usepackage[utf8]{inputenc}
\usepackage[english]{babel}
\usepackage[T1]{fontenc}%
\usepackage{amsthm,amsmath,amsfonts,amssymb,mathtools,stmaryrd,mleftright,bbm}
\usepackage{cancel} %
\usepackage{hyperref} %
\usepackage[dvipsnames,table]{xcolor} %
\usepackage{enumitem} %
\usepackage{verbatim} %
\usepackage{wrapfig}

\usepackage{algorithm}
\usepackage{algpseudocode}

\newcommand{\Input}{\item[\textbf{Input:}]}
\newcommand{\Output}{\item[\textbf{Output:}]}

\usepackage{caption}
\usepackage{subcaption}

\newcommand{\LL}{\mathcal{L}}
\newcommand{\Pow}{\mathcal{P}}
\newcommand{\N}{\mathcal{N}}
\newcommand{\R}{\mathbb{R}}

\newcommand{\eps}{\varepsilon}

\DeclareMathOperator*{\E}{E}
\DeclareMathOperator*{\Var}{Var}
\DeclareMathOperator*{\Cov}{Cov}

\newtheorem{thm}{Theorem}
\newtheorem*{thm*}{Theorem}
\newtheorem{Prop}[thm]{Proposition}
\newtheorem{Lem}[thm]{Lemma}

\theoremstyle{definition}
\newtheorem{example}[thm]{Example}
\newtheorem{Def}[thm]{Definition}
\newtheorem*{continuancex}{Example \continuanceref{} Continued}
\newenvironment{continuance}[1]
  {\newcommand\continuanceref{\ref{#1}}\continuancex}
  {\endcontinuancex}

\usepackage{todonotes}

\definecolor{interaction-green}{HTML}{64C285}
\definecolor{dependence-purple}{HTML}{946EE6}
\definecolor{cross-pred-purple}{HTML}{C862CC}
\definecolor{loco-red}{HTML}{D98068}
\definecolor{total-yellow}{HTML}{D9CD48}
\definecolor{cov-purple}{HTML}{6285CC}
\definecolor{v1gray}{RGB}{169,169,169}
\definecolor{v2gray}{RGB}{128,128,128}
\definecolor{vCgray}{RGB}{211,211,211}

\hypersetup{
  colorlinks=true,
  citecolor=black,
  linkcolor=black,
  urlcolor=blue
}

\usepackage{lipsum}

\newcommand\blfootnote[1]{%
  \begingroup
  \renewcommand\thefootnote{}\footnote{#1}%
  \addtocounter{footnote}{-1}%
  \endgroup
}

\begin{document}

\twocolumn[

\aistatstitle{Disentangling Interactions and Dependencies in Feature Attribution}

\aistatsauthor{Gunnar König* \And  Eric Günther* \And Ulrike von Luxburg}

\aistatsaddress{University of Tübingen\\ Tübingen AI Center \And University of Tübingen\\ Tübingen AI Center \And University of Tübingen\\ Tübingen AI Center} ]

\begin{abstract}
In explainable machine learning, global feature importance methods try to determine how much each individual feature contributes to predicting the target variable, resulting in one importance score for each feature. But often, predicting the target variable requires interactions between several features (such as in the XOR function), and features might have complex statistical dependencies that allow to partially replace one feature with another one. In commonly used feature importance scores these cooperative effects are conflated with the features' individual contributions, making them prone to misinterpretations. 
In this work, we derive DIP, a new mathematical decomposition of individual feature importance scores that disentangles three components: the standalone contribution and the contributions stemming from interactions and dependencies. We prove that the DIP decomposition is unique and show how it can be estimated in practice. Based on these results, we propose a new visualization of feature importance scores that clearly illustrates the different contributions.
\end{abstract}

\section{INTRODUCTION}

\begin{figure}[t] 
    \centering
    \includegraphics[width=0.98\linewidth]{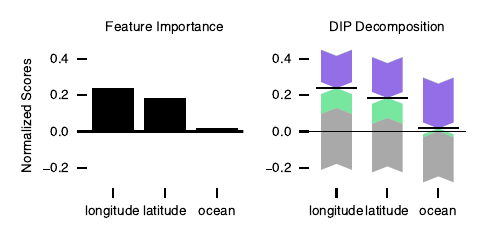}
    \caption{{\bf Feature importance, old vs. new. } Consider a model that predicts house prices using the features 
    \textit{longitude}, \textit{latitude}, and \textit{ocean proximity}. \emph{Left:} Leave-One-Covariate-Out (LOCO) scores. \emph{Right:} Our decomposition of the same scores (\textcolor{black}{black}) into each feature's standalone contribution (\textcolor{v2gray}{gray}) and the contributions of interactions (\textcolor{interaction-green}{green}) and dependencies (\textcolor{dependence-purple}{purple}). The arrows of the bars indicate whether the contribution is positive or negative; their values sum up to the LOCO scores.
}
\label{fig:intro:motivating-example}
\end{figure}

\blfootnote{*Equal contribution. Preprint.}
Tools from explainable AI (xAI) are increasingly employed not only to explain a machine learning (ML) model's mechanism but also to gain insight into the data generating process (DGP) \citep{freiesleben2024scientific}. Global loss-based feature importance techniques are often used to learn about the features' predictive power, that is, their ability to accurately predict the underlying target. To enable such insight, the methods remove features from the dataset, for example, by marginalizing out variables or refitting the model, and quantify the global effect on the empirical risk \citep{covert2021explaining,molnar2022interpretable}. This contrasts methods like SHAP \citep{lundberg2017unified}, which explain how a specific model arrives at its prediction for a particular observation.
\\
Existing feature importance methods try to explain the features' joint predictive power with just one individual score for each feature. This is problematic since, commonly, the predictive power is not simply the sum of the features' standalone contributions but also the result of cooperative forces: Interactions between several variables might unlock additional predictive power, and variable dependencies might render the different standalone contributions redundant. As such, when attributing the predictive power with just one score per feature, the individual scores conflate standalone and cooperative contributions, making them prone to misinterpretation.\\
The main contribution of this paper is to derive a new mathematical decomposition of individual feature importance scores that disentangles the contributions of individual features and cooperative effects stemming from interactions and dependencies.\\
Let us consider an illustrative example. Suppose the goal is to predict the price of a house based on its longitude, latitude, and proximity to the ocean.
Inspecting a typical feature importance plot, the relevance of feature cooperation does not become clear (Figure \ref{fig:intro:motivating-example}, left). Using our method (Figure \ref{fig:intro:motivating-example}, right), we can see that the features longitude and latitude have limited use alone but are highly predictive if combined with the remaining features via an \textcolor{interaction-green}{\textit{interaction}} (since together they determine the exact location). We can also see that the feature ocean proximity is useful alone, but its predictive power can be replaced due to its \textcolor{dependence-purple}{\textit{dependence}} with the remaining variables (longitude and latitude can replace ocean proximity). For details on the implementation we refer to Appendix \ref{appendix:estimation-implementation}.\\
In Section \ref{sec:cooperative-effect}, we show that the impact of interactions and dependencies are entangled in the predictive power. To disentangle them, we first separate pure interactions and main effects \emph{in the predictor} (Section \ref{sec:model-decomposition}). Knowing how to decompose a predictor, we can decompose the \textit{predictive power} of two groups of features as the sum of their respective standalone contributions and the contributions of between-group cooperation via interactions and dependencies (Section \ref{sec:cooperative-effect-decomposition}). The decomposition can explain the outputs of popular feature importance methods such as LOCO or SAGE (Section \ref{sec:decomposing-feature-importance}). We demonstrate its utility on real-world data in Section \ref{section: experiments}.\\
In contrast to existing approaches that explain the relevance of interactions for a prediction function \citep{lundberg2018consistent,sundararajan2020shapley,bordt2023shapley,herbinger2023decomposing}, we focus on learning about the relationships in the data. Thus, we explain the relevance of interactions for the predictive power instead of a specific model's mechanism; we avoid marginal sampling -- crucial to enable insight into the data  \citep{chen2020true,hooker2021unrestricted,freiesleben2024scientific}; and we explain the cooperative contributions of \emph{both} interactions \emph{and dependencies}.
\paragraph{Contributions}{
\begin{itemize}
  \item We propose DIP (\textbf{d}isentangling \textbf{i}nteractions and de\textbf{p}endencies), a unique decomposition of $\LL^2$-loss based predictive power that explains the contributions of both interactions and dependencies between two groups, as well as new plots that clearly visualize their respective contribution (Section \ref{sec:cooperative-effect-decomposition}).
  \item In Section \ref{sec:cooperative-effect}, we show that predictive power conflates interactions and dependencies. To disentangle their contributions, we show how to uniquely separate main effects and so-called pure interactions \textit{in an ML model} in Section \ref{sec:model-decomposition}.
  \item We show that the decomposition can be used to explain popular feature importance techniques (Section \ref{sec:decomposing-feature-importance}) and demonstrate the method's practical usefulness on real-world data (Section \ref{section: experiments}). A \texttt{python} implementation of the method is available on GitHub (\url{https://github.com/gcskoenig/dipd}).
\end{itemize}
}
\section{RELATED WORK}
\label{sec:related-work}
\paragraph{Explanation Techniques that Attribute Interactions}
There is a large amount of literature on explainable AI, we refer to \cite{molnar2022interpretable} for an overview.
Here, we focus on techniques that attribute interactions.
Shapley interaction values \citep{grabisch1999axiomatic, lundberg2018consistent} and higher order variants of them \citep{sundararajan2020shapley,herren2022statistical,bordt2023shapley,hiabu2023unifying} are local attribution methods that attribute interactions. A range of methods for their estimation has been proposed \citep{fumagalli2024kernelshap,fumagalli2024shap,muschalik2024beyond}.
Friedman's H-statistic \citep{friedman2008predictive} or GADGET \citep{herbinger2023decomposing} are global alternatives. In contrast to existing methods, we avoid marginal sampling techniques, which do not allow insight into the data \citep{hooker2021unrestricted,freiesleben2024scientific}. Furthermore, we decompose the predictive power instead of explaining a model's predictions and disentangle the contributions of interactions and dependencies.
\paragraph{Functional Decomposition}
The generalized functional ANOVA (generalized fANOVA) decomposition \citep{hooker2007generalized} decomposes a function into components of different interaction order. Its computation is generally hard \citep{li2012general,lengerich2020purifying}.
The decomposition lays the foundations for generalized Sobol indices \citep{chastaing2015generalized, gao2023probabilistic}. While the aforementioned methods attribute every possible subset of features, we focus on explaining the cooperation between two groups of features. Our decomposition is comparatively easy to estimate and interpretable. Also, we attribute both interactions and dependencies.
\paragraph{Partial Information Decomposition}
When replacing the $\LL^2$-loss with the cross-entropy-loss, the cooperative impact becomes the interaction information (\citealp{covert2020understanding}, Appendix C).
The problem of decomposing the interaction information into a redundancy and a synergy component is called partial information decomposition and is discussed in \cite{williams2010nonnegative,barrett2015exploration,griffith2015quantifying,kolchinsky2022novel}.
\paragraph{Communality Analysis}{
In a similar vein, a range of work has focused on decomposing the explained variance of linear regression models into the standalone and shared contributions of the features \citep{seibold1979commonality,nathans2012interpreting,ray2014using}. %
We generalize existing results to ML models, accounting for nonlinearities and interactions.
}
\section{BACKGROUND}\label{section: background}
\subsection{Notation}
Throughout the paper, we consider $(X,Y)\sim P$ to be our data generating process (DGP), consisting of two random variables: the features $X=(X_1,...,X_d)$ in $\R^d$ and the labels $Y$ in $\R$. They are sampled from some probability measure $P$ on $\R^d\times\R$. We assume $X_1,...,X_d,Y$ as well as every prediction function to be $\LL^2$-measurable with respect to $P$. %
We denote the set of all features by $D:=\{1,...,d\}$ and its power set by $\Pow(D)$. For a set of features $J\subseteq D$, the term $\bar{J}$ refers to the set $D\setminus J$. For sets of just one feature, we tend to drop the brackets for readability, for example, $\bar{j}$ instead of $\overline{\{j\}}$. %
For functions $f:\R^d\to \R$, we often write $\E(f)$ instead of $\E(f(X))$ for better readability, likewise for $\Var$ and $\Cov$. 
By a \textbf{Generalized Additive Model (GAM)} we mean a function $g:\R^d\to \R,\; g(X)=g_1(X_1) +...+ g_d(X_d)$ that can be written as a sum of functions depending on only one feature \citep{hastie1986generalized}. We use the term \textbf{Generalized Groupwise Additive Model (GGAM)} in $X_S$ and $X_T$, where $S,T\subseteq D$, for a function that can be written in the form $g:\R^d\to \R,\; g(X)=g_S(X_S)+g_T(X_T)$. We refer to the component functions of a GAM (GGAM) as \textbf{main effects}. A function that cannot be written as a GAM (GGAM in $X_S$ and $X_T$) is called an \textbf{interaction} (interaction between $X_S$ and $X_T$).  
\subsection{Loss-Based Feature Importance}
By loss-based feature importance we mean methods that quantify the relevance of features by comparing the predictive power of subsets of features \citep{breiman2001random,strobl2008conditional,lei2018distribution,covert2020understanding,konig2021relative,williamson2021nonparametric}.  %
To measure the predictive power of a set of features for a specific model $f$ and loss $L$, we follow the notation of \cite{covert2020understanding} and introduce a \textbf{value function} $v_{f,L}:\Pow(D)\to \R_{\geq 0}$, which measures the drop in risk when knowing $X_S$ compared to having access to none of the features. More formally,
\begin{align*}
  v_{f, L}(S) := E(L(f_\emptyset, Y)) - E(L(f_S(X_S), Y)),
\end{align*}
where $f_S$ is a \textbf{restricted function} that only has access to the features $S \subseteq D$. 
This restricted function can be computed by integrating out the unused features $D \setminus S$ using either the marginal or conditional expectation. Given that our goal is to learn about the DGP, which marginal sampling is unsuitable for \citep{freiesleben2024scientific}, we always use the conditional version
\begin{equation*}
  f_S(x_S) = E(f(X)\mid X_S=x_S).
\end{equation*}
Focusing on regression and on understanding a DGP rather than a particular predictor, we study the $\LL^2$-loss of the optimal predictor $f^*(x)=\E(Y\mid X=x)$ for that DGP. In this setting, we drop the indices of the value function. It can be shown that the respective value function satisfies
\begin{align*}
    v(S):&= v_{f^*, \LL^2}(S) = \Var(Y) - \E\left((Y-f^*_S(X_S))^2\right)\\
    &= \Var(\E(Y\mid X_S)),
\end{align*}
corresponding to the explained variance of $Y$ conditional on $X_S$ (\citealp{covert2020understanding}, Appendix C). We call $v(S)$ the \textbf{predictive power} of the features $S$. %
We denote the normalized version as $\bar{v}(S) := v(S)/\Var(Y)$.\\

In this setting, $f^*_S$, based on conditional expectation, can also be estimated by refitting the model with just $X_S$ \citep{lei2018distribution,williamson2021nonparametric}.  
The Leave-One-Covariate-Out (LOCO) method \citep{lei2018distribution,williamson2021nonparametric} uses refitting to compute $v(D)-v\left(\bar{j}\right)$, which is the drop in predictive power when removing feature $j$ from $D$.

\section{PREDICTIVE POWER DOES NOT REVEAL COOPERATION}
\label{sec:cooperative-effect}
Throughout the paper, we develop a method that explains the relevance of cooperation via interactions and dependencies for the predictive power. In this section, we show that even access to the predictive power of all subsets of  features -- the basis of feature importance methods -- is not sufficient to solve this task.\\
We start the section by introducing what we call the \textit{cooperative impact}, that is, the effect of cooperations on the predictive power. We show that the cooperative impact results from two forces, interactions and dependencies, but that it may not reveal their relevance since their effects on the predictive power might cancel out. Later in the paper we show how to estimate pure interactions in a model (Section \ref{sec:model-decomposition}), which will allow us to disentangle the two cooperative forces in the cooperative impact (Section \ref{sec:cooperative-effect-decomposition}).
\paragraph{The Impact of Cooperations On Predictive Power.}{%
We start by defining the cooperative impact.
\begin{Def}[\textbf{Cooperative Impact}]\label{def:cooperative-impact}
Let $(X,Y)\sim P$ be a DGP on $\R^d\times \R$ and $J\subseteq D$ a subset of features. The cooperative impact $\Psi$ of $J$ and $\bar{J}$ is defined as
$$ \Psi\left(J,\bar{J}\right) := v\left(J \cup \bar{J}\right) - \left(v(J) + v\left(\bar{J}\right)\right).$$
\end{Def}
The cooperative impact results from \textit{two cooperative forces}: interactions and dependencies between the features. As the housing price example in the introduction showed, interactions and dependencies can affect the joint predictive power. They can unlock joint predictive information that is otherwise unavailable or induce redundancies that reduce the joint contribution. However, in their absence, the joint contribution is simply the sum of the features' standalone contributions, and the cooperative impact is zero (Proposition \ref{prop:the-only-coop-effects}, proof in Appendix \ref{app: proof main theorem}).
\begin{Prop}[\textbf{Without Interactions or Dependencies, the Cooperative Effect is Zero}]\label{prop:the-only-coop-effects}
Let $(X,Y)\sim P$ be a DGP and $J\subseteq D$ a subset of features. If $X_J$ and $X_{\bar{J}}$ are independent and the $\LL^2$-optimal predictor is a GGAM $g^*=g^*_J+g^*_{\bar{J}}$ in $X_J$ and $X_{\bar{J}}$, then 
$\Psi\left(J,\bar{J}\right) = 0.$
\end{Prop}
\paragraph{The Cooperative Impact May Not Reveal the Importance of Cooperations.}{Although the cooperative impact is zero if no interactions or dependencies are present, the converse is not true. %
We illustrate this issue with an example.
\begin{example}[\textbf{Contributions of Interactions and Dependencies Cancel Out}]\label{example:cooperative-forces-cancel-out}
Let $Y=X_1 + X_2 + cX_1X_2$, where $X$ is normally distributed on $\R^2$ with $\Var(X_1) = \Var(X_2) = 1$ and $\Cov(X_1, X_2) = \beta$.\\
In DGP 1, we set $c=\beta=0$, meaning there are no interactions or dependencies, and thus no cooperations. In DGP 2, we set $c=\sqrt{6}$ and $ \beta = 0.5$ such that there is cooperation both in the form of interactions and dependencies. In both cases we obtain $\bar{v}(1\cup 2) = 1$, $\bar{v}(1)=0.5$, and $\bar{v}(2)=0.5$.
Hence $\Psi(1,2) = 0$. See Appendix \ref{app:example:linear} for the formal derivation.
\end{example}
In the example, there is no cooperation in the first DGP, but the variables cooperate both in form of interactions and dependencies in the second DGP. 
Nevertheless, the value functions for all possible sets of features are the same in both DGPs; the cooperative impact is zero. The reason is that in the second DGP the positive effect of the interaction and the negative effect of the redundancy-inducing dependence cancel each other out.\\
The example shows that value functions are not sufficient to quantify the relevance of interactions and dependencies. To reveal their impact, we first estimate pure interactions in a model.
}

\section{ESTIMATING PURE INTERACTIONS}
\label{sec:model-decomposition}
In this section, our goal is to separate interactions and main effects in a model $f$. Given two groups of variables, we decompose $f$ into an \textit{interaction term} that only represents interactions \textit{between} the groups, and \textit{main effects} that only permit interactions \textit{within} groups.

\paragraph{Characterizing Pure Interactions Using Additive Models}{
More formally, given two groups of features $J$ and $\bar{J}$, we want to decompose a prediction function $f$ as
$$f(x) = g_J(x_J) + g_{\bar{J}}(x_{\bar{J}}) + h(x),$$
where $g$ is the model that only permits within-group interactions, that is, a generalized groupwise additive model (GGAM) of the form $g = g_J + g_{\bar{J}}$. %
Moreover, we want the remaining interaction term $h$ to be ``pure'', meaning that everything that can be expressed without between-group interactions should be represented in the GGAM $g$.
This intuition gives rise to the following definition.
\begin{Def}[\textbf{Pure Interaction}]
\label{def:pure-interaction}
    Let $f:\R^d\to\R$ be a %
    function on some probability space $\left(\R^d, P\right)$ and $J\subseteq D$ a subset of features. Let further $g$ be the $\LL^2$-optimal approximation of $f$ within all GGAMs in $X_J$ and $X_{\bar{J}}$. %
    We define the \emph{pure interaction} of $f$ with respect to $X_J$ and $X_{\bar{J}}$ as $f-g$. 
\end{Def}
In short, the pure interaction represents what cannot be explained by a GGAM. This characterization of pure interactions between groups is not only intuitive; in the following paragraphs, we show that it is also unique, entails a simple estimation procedure, and overlaps with existing definitions. 
As we will see in Section \ref{sec:cooperative-effect-decomposition}, its properties allow finding an interpretable decomposition of the predictive power.
\paragraph{Uniqueness of the GGAM Components}
We show that under mild assumptions, the component functions $g_J$ and $g_{\bar{J}}$ of the GGAM $g$ in Definition \ref{def:pure-interaction} are unique, see Appendix \ref{app: uniqueness GAM}. This ensures that our decomposition $f=g_J + g_{\bar{J}} + h$ is unique as well, which will later imply the uniqueness of the DIP decomposition.
\paragraph{Estimation}{
Our definition of pure interactions directly entails a procedure for their estimation. To find the pure interaction in $f$ with respect to two groups $J$ and $\bar{J}$, we only need to fit one GGAM $g$ approximating $f$ to get the pure interaction as the residual $h = f - g$. Thereby, we can leverage a broad range of methods and implementations \citep{hastie1986generalized,serven2018pygam,nori2019interpretml}.\\
In our context, we want to decompose $f^*$, and therefore approximate $f^*$ with a GGAM. 
As Lemma \ref{app: Lem: tower property} in Appendix \ref{app: lemma subsection} shows, we can equivalently approximate $Y$ using the GGAM. Since only an approximation of $f^*$ is available in practice, we fit the GGAM on $Y$ in our experiments.
}
\paragraph{Properties and Relation To Existing Definitions} We can equivalently characterize pure interactions as terms that are not approximable by only one of the two groups of variables, as the following theorem shows (proof in Appendix \ref{app: proof equivalence pure interaction}).

\begin{thm}[\textbf{Equivalent Characterization of Pure Interactions}]\label{Thm: GAM residual gives pure interaction}
    Let $f:\R^d\to\R$ be a %
    function on some probability space $\left(\R^d, P\right)$ and $J\subseteq D$ a subset of features. A GGAM $g=g_J+g_{\bar{J}}$ is the $\LL^2$-optimal approximation of $f$ within all GGAMs in $X_J$ and $X_{\bar{J}}$ if and only if the residual $h:=f-g$ satisfies $\E(h\mid X_J)=0$ and $\E(h\mid X_{\bar{J}})=0.$
\end{thm}
Theorem \ref{Thm: GAM residual gives pure interaction} enables another interpretation of pure interactions. In particular, $\E(h\mid X_J)=\E(h\mid X_{\bar{J}})=0$ implies that pure interactions in the optimal predictor $f^*$ for a DGP do not contribute to the standalone predictive powers $v(J)$ and $v\left(\bar{J}\right)$.\\
Furthermore, Theorem \ref{Thm: GAM residual gives pure interaction} implies that in the two-dimensional setting, our definition of pure interactions coincides with the one in \cite{lengerich2020purifying}, which is based on \cite{hooker2007generalized}.\\
We highlight that pure interactions depend not only on the function $f$ but also on the dependencies between features; see Appendix \ref{app:example:linear} for an example.

\section{DISENTANGLING THE CONTRI- BUTIONS OF INTERACTIONS AND DEPENDENCIES}\label{sec:cooperative-effect-decomposition}
Equipped with the tools to decompose a model $f$, we are ready to \textbf{d}isentangle the effects of \textbf{i}nteractions and de\textbf{p}endencies on the predictive power (DIP decomposition). More precisely, we decompose the cooperative impact (see Definition \ref{def:cooperative-impact}) in Theorem \ref{thm: decomposition for groups}. For the proof we refer to Appendix \ref{app: proof main theorem}.
\begin{thm}[\textbf{Cooperative Impact Decomposition}]\label{thm: decomposition for groups}
    Let $(X,Y)\sim P$ be a DGP on $\R^d\times\R$, $J\subseteq D$ a subset of features, $f^*=\E(Y\mid X)$ the $\LL^2$-optimal predictor and $g^*=g^*_J+g^*_{\bar{J}}$ the $\LL^2$-optimal GGAM in  $X_J$ and $X_{\bar{J}}$. We call $h^*:=f^*-g^*$. Then, we get a decomposition

 \begin{align*}
        \Psi\left(J, \bar{J}\right) 
        =&\textcolor{interaction-green}{\underbrace{\Var(h^*)}_{\shortstack{\scriptsize \text{Interaction} \\ \scriptsize \text{Surplus}}}} - \textcolor{dependence-purple}{\underbrace{\mathrm{Dep}\left(J,\bar{J}\right)}_{\shortstack{\scriptsize \text{Main Effect} \\ \scriptsize \text{Dependencies}}}}, \quad \text{where}\\
        \textcolor{dependence-purple}{\mathrm{Dep}\left(J,\bar{J}\right)}  
        :=& \textcolor{cross-pred-purple}{
                   \underbrace{\Var\left (\E\left (g^*_J\mid X_{\bar{J}}\right )\right ) + \Var\left (\E\left (g^*_{\bar{J}}\mid X_J\right )\right) 
                   }_{\text{Cross-Predictability}}} \\
        &+ \textcolor{cov-purple}{
                   \underbrace{2\Cov\left (g^*_J, g^*_{\bar{J}}\right )
                   }_{\text{Covariance}}}.
    \end{align*}
\end{thm}
\begin{figure*}[ht]
\centering
    \begin{subfigure}[t]{0.2\textwidth}
        \centering
        \includegraphics[width=0.99\linewidth]{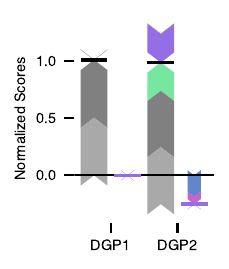}
        \caption{Example \ref{example:cooperative-forces-cancel-out}}
        \label{fig:forceplot-coop-forces-cancel-out}
    \end{subfigure}
    \hfill
    \begin{subfigure}[t]{0.25\textwidth}
    \centering
    \includegraphics[width=0.44\textwidth]{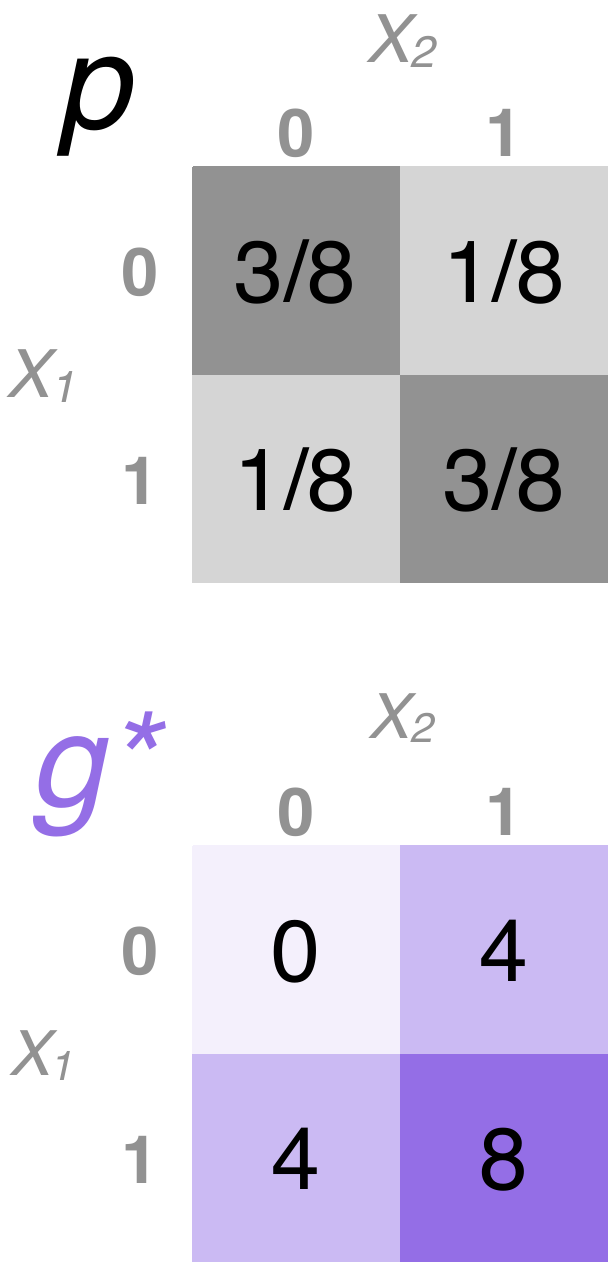}
    \hfill
    \includegraphics[width=0.5\textwidth]{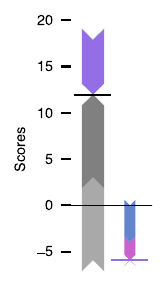}
    \caption{Example \ref{Example: student redundancy}}
    \label{fig: student redundancy}
    \end{subfigure}
    \begin{subfigure}[t]{0.25\textwidth}
    \centering
    \includegraphics[width=0.44\textwidth]{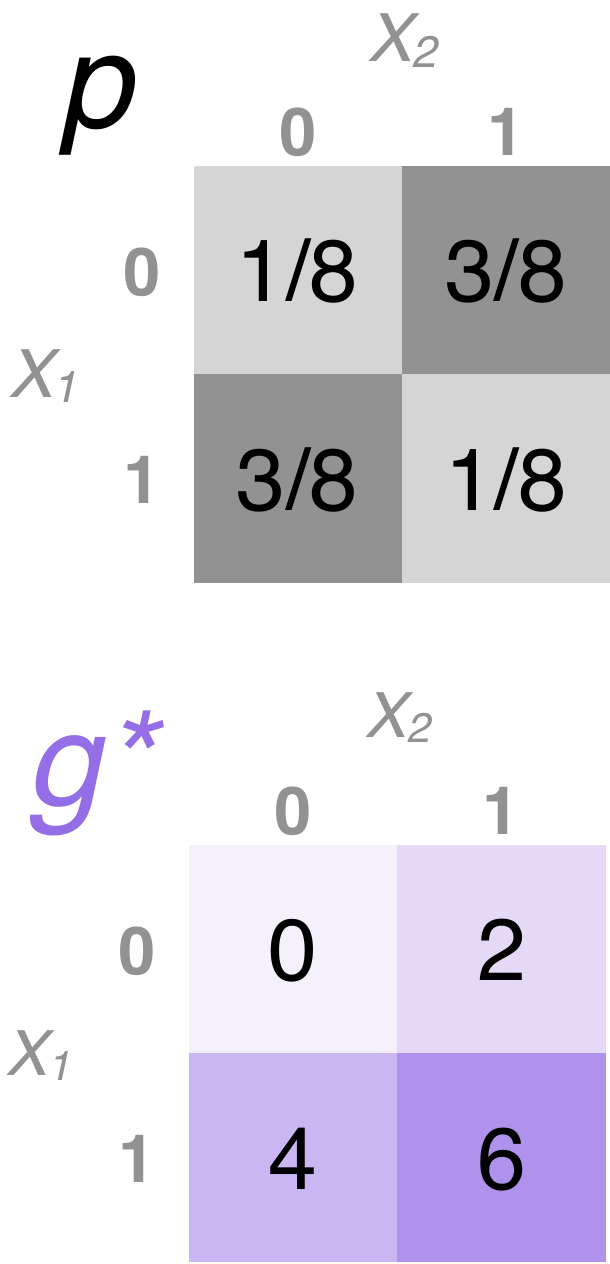}
    \includegraphics[width=0.5\textwidth]{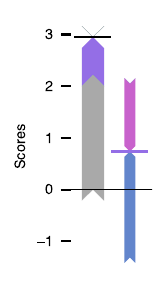}
    \caption{Example \ref{example: student enhancement}}
    \label{fig: student enhancement}
    \end{subfigure}
    \begin{subfigure}[t]{0.25\textwidth}
    \centering
    \includegraphics[width=0.44\textwidth]{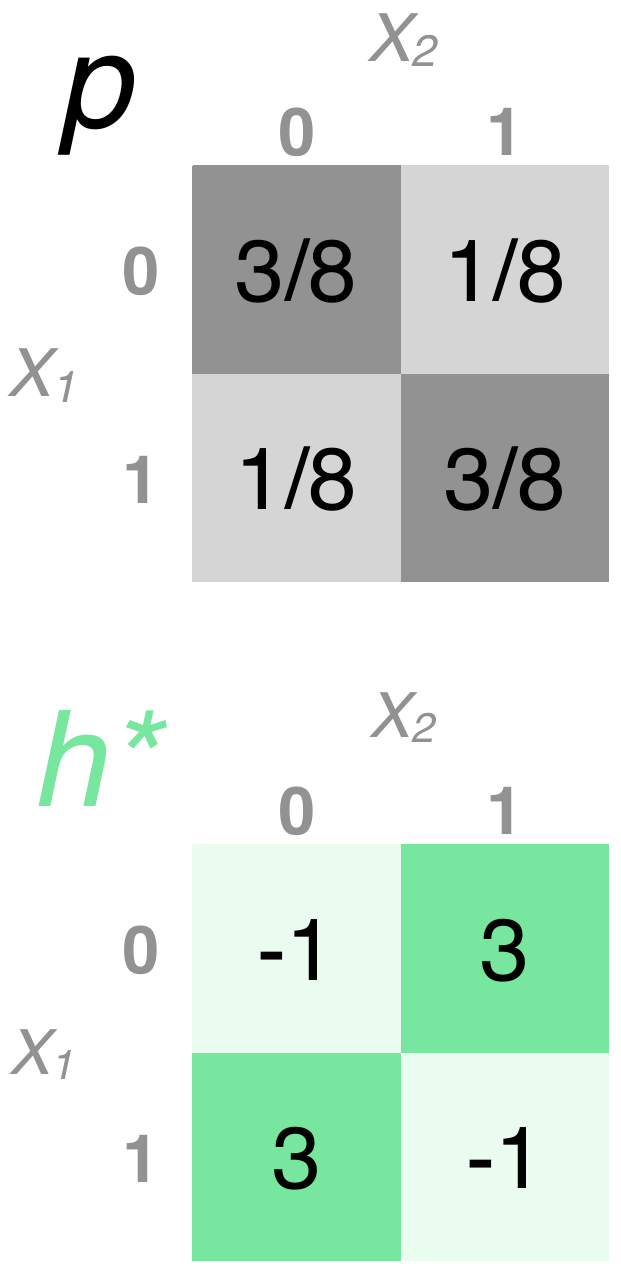}
    \includegraphics[width=0.5\textwidth]{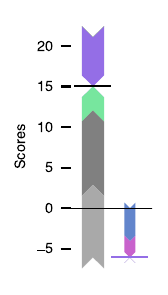}
    \caption{Example \ref{example: binary interaction}}
    \label{fig: binary interaction}
    \end{subfigure}

\caption{\textbf{Examples.} For each example, we show a forceplot visualizing the DIP decomposition into standalone contributions \textcolor{v1gray}{($v(1)$} and \textcolor{v2gray}{$v(2)$}), \textcolor{dependence-purple}{main effect dependencies ($\mathrm{Dep}(1,2)$)} and \textcolor{interaction-green}{interaction suplus}, where the direction of each bar (upward or downward) represents the sign. They sum up to $v(1,2)$ (black horizontal line). The slim bars (right) show the decomposition of \textcolor{dependence-purple}{$\mathrm{Dep}(1,2)$} (purple horizontal line) into \textcolor{cov-purple}{covariance} and \textcolor{cross-pred-purple}{cross-predictability}. For Examples \ref{Example: student redundancy}-\ref{example: binary interaction}, we additionally show heatmaps visualizing the distribution (top) and $g^*$ or $h^*$ (bottom).}
\end{figure*}
\paragraph{Interpretation}
First, we recall that both pure interactions and main effects are unique (Section \ref{sec:model-decomposition}, Appendix \ref{app: uniqueness GAM}), and thus the decomposition is, too.\\
Next, we highlight that the cooperative impact decomposes into a term that only depends on the pure interaction $h^*$ %
 and terms that only depend on the main effects $g^*$%
. In other words, the contributions of interactions and main effects simply add up. This is no coincidence but a direct consequence of the properties of pure interactions (Definition \ref{def:pure-interaction}). Roughly, the properties ensure that any covariance terms involving $h$ and $g$ vanish (cf. the proof of Theorem \ref{thm: decomposition for groups}).\\
The effect of interactions on the predictive power is given by the variance of the pure interaction term $h^*$. Since the variance cannot be negative, we refer to it as the \textcolor{interaction-green}{\textbf{interaction surplus}}. It measures how much of the joint predictive power can only be explained with between-group interactions.\\
We notice that $\textcolor{dependence-purple}{\mathrm{Dep}\left(J, \bar{J}\right)}$ vanishes if the two groups of features $X_J$ and $X_{\bar{J}}$ are independent. It measures how much of the cooperative impact is caused by \emph{dependencies}. More precisely, by the dependencies between the effects of $X_J$ and $X_{\bar{J}}$ on $Y$, given by the main effects $g^*_J$ and $g^*_{\bar{J}}$. Thus we refer to $\mathrm{Dep}\left(J, \bar{J}\right)$ as \textcolor{dependence-purple}{\textbf{main effect dependencies}}. As we will see, the main effect dependencies can have a positive or negative influence on the cooperative impact. The main effect dependencies consist of two parts: We refer to them as the main effect cross-predictability and the main effect covariance.\\ 
Intuitively, the \textcolor{cross-pred-purple}{\textbf{main effect cross-predictability}} quantifies how redundant the contributions of the two groups of features are. More formally, it measures how much of the variance of each main effect could also be explained by the respective other group of variables. Being a sum of variances, the cross-predictability is always positive, and its impact on the cooperative impact is always negative. This is consistent with the intuition that the joint predictive power should decrease if the variables share more variation.\\
Unlike the cross-predictability, the \textcolor{cov-purple}{\textbf{main effect covariance}} can be either positive or negative. If it is negative and its absolute value outweighs the cross-predictability, it increases the cooperative impact $\Psi\left(J,\bar{J}\right)$. This may seem counter-intuitive because then, the two groups of variables are more predictive together than individually, even if they do not interact. For an intuitive example on how a negative main effect covariance induces this improvement in predictive power, we refer to Example \ref{example: student enhancement} below. This phenomenon can also be observed in multivariate linear models and is tied to the word ``enhancement'' \citep{friedman2005graphical} or ``suppression'' \citep{shieh2006suppression}. We extend the analysis of this phenomenon to the more general setting of ML involving GAMs.\\
 Note that $\mathrm{Dep}\left(J,\bar{J}\right)$ and its two components cannot simply be determined from the dependencies between features. Instead they explain the relevance of dependencies \textit{for the predictive power} (details in Appendix \ref{app:examples:correlation vs main effect dependencies} and \ref{app:examples: zero main order dependencies}).
\paragraph{Estimation}{
In practice, the optimal models $f^*$ and $g^*$ are not available, but can be approximated by ML models. To avoid bias due to overfitting, the scores can be reformulated in terms of empirical risk on test data (details in Appendix \ref{appendix:estimation}). %
We highlight that only three additional model fits are required for the decomposition, irrespective of the group size. %
}
\paragraph{Illustrative Examples}
First, we revisit Example \ref{example:cooperative-forces-cancel-out}, where we can now reveal the relevance of cooperation (derivations in Appendix \ref{app:example:linear}). Then we illustrate the interpretation of DIP in Examples \ref{Example: student redundancy} to \ref{example: binary interaction} (derivations in Appendix \ref{app:examples:student examples}).
\begin{continuance}{example:cooperative-forces-cancel-out}[\textbf{Cooperative Forces May Cancel Out}, Figure \ref{fig:forceplot-coop-forces-cancel-out}]
For DGP 1 ($c=\beta=0$), the interaction surplus and the main effect dependencies both vanish, as one would expect; see Figure \ref{fig:forceplot-coop-forces-cancel-out}, left bar. On the other hand, for DGP 2 $\left(c=\sqrt{6}, \beta = 0.5\right)$ we get a (normalized) interaction surplus of $0.26$, cross-predictability of $0.07$ and covariance of $0.19$, 

such that interaction surplus and main effect dependencies cancel out (Figure \ref{fig:forceplot-coop-forces-cancel-out}, right bar).
\end{continuance}
\begin{example}[\textbf{Negative Cooperative Impact via Dependence}, Figure \ref{fig: student redundancy}]\label{Example: student redundancy}
Suppose we want to predict a student's points $Y$ in an exam based on two binary features that indicate whether a student did their homework $(X_1=1)$ or not $(X_1 = 0)$ and whether a student studied for at least 10 hours $(X_2=1)$ or not $(X_2=0)$. We assume the two features are $\mathrm{Ber}\left(0.5\right)$-distributed and positively correlated with $P(X_1=X_2)=0.75$ (see Figure \ref{fig: student redundancy}, top left), because doing homework requires time.
Assume the points follow the rule $Y= 4X_1 + 4X_2$ (see Figure \ref{fig: student redundancy}, bottom left), meaning there are no interactions, so $h^*=0$. This results in a negative cooperative impact $\Psi(1,2)=-6$ (\textcolor{v1gray}{$v(1)=9$}, \textcolor{v2gray}{$v(2)=9$}, and $v(1\cup 2) =12$). The DIP decomposition delivers a \textcolor{cross-pred-purple}{cross-predictability of $2$} and a \textcolor{cov-purple}{main effect covariance of $4$}, whose negatives sum up to the observed cooperative effect of $\Psi(1,2)=-6$. %
\end{example}
Intuitively, this negative cooperative impact makes sense since the two features are correlated and can partially replace each other. This leads to either variable being able to recover more than half of the explained variance of $Y$ on its own and thus $\Psi(1,2) < 0$. In short the negative effect of the main effect dependencies indicates that the two variables have similar information about the target.
\begin{example}[\textbf{Positive Cooperative Impact via Dependence}, Figure \ref{fig: student enhancement}]\label{example: student enhancement}
In our second example, $Y$ still reflects the points and $X_1$ whether a student did homework, but feature 2 now indicates whether the student attended the review session $(X_2=1)$ or not $(X_2=0)$. Students who did their homework tend not to attend the review session, so we again choose $X_1$ and $X_2$ to be $\mathrm{Ber}\left(0.5\right)$-distributed but this time negatively correlated with $P(X_1=X_2)=0.25$, see Figure \ref{fig: student enhancement}, top left.
The relationship between points and features is given by $Y=4X_1+2X_2$
(see Figure \ref{fig: student enhancement}, bottom left), so again $h^*=0$. This time, we get a positive cooperative impact (\textcolor{v1gray}{$v(1)=2.25$}, \textcolor{v2gray}{$v(2)=0$}, and $v(1\cup 2)=3$). The \textcolor{cross-pred-purple}{cross-predictability is $1.25$} and the \textcolor{cov-purple}{main effect covariance $-2$}, outweighing the former. %
\end{example}
Why do the two features have more predictive power together than individually, despite the absence of any interaction? Although attending the review session adds two points, we have $v(2)=0$. Indeed, both columns of the lower table in Figure \ref{fig: student enhancement} have a weighted average of 3, that is, $\E(Y\mid X_2=0)=\E(Y\mid X_2=1) = 3$. So, knowing solely $X_2$ does not help predicting $Y$. This is due to the correlation. For students who attended the review session, it is more likely that they did not do homework, which is bad for their score. This cancels out the positive effect of the review session. One could say the correlation works against the prediction. Once we use both features, the predictive power of $X_2$ is revealed because $X_1$ is known. The lower table of Figure \ref{fig: student enhancement} again illustrates this well: Once we can distinguish the rows, the difference between the two columns becomes visible. \\
The crucial part of Example \ref{example: student enhancement} is that the values of $X_1$ and $X_2$ that are likely to occur concurrently have opposing effects on $Y$. Formally, that is $\Cov(g^*_1, g^*_2)<0$. The example illustrates how a negative main effect covariance can improve the features' joint predictive power compared to their individual ones.\footnote{Note that this phenomenon does not need one of the two value functions to vanish. We simply chose an example with a vanishing value function for better illustration. %
}
\begin{example}[\textbf{Interactions}, Figure \ref{fig: binary interaction}]\label{example: binary interaction}
    Let us again consider the positively correlated $\mathrm{Ber}\left(0.5\right)$-distributed variables $X_1$ and $X_2$ from Example \ref{Example: student redundancy} that satisfy $P(X_1=X_2)=0.75$. This time, we set
    $Y= 8(X_1\vee X_2) - 1$,
    where $\vee$ denotes the logical OR-operator. This may be rewritten as as 
    $Y=4X_1+4X_2+4(X_1\oplus X_2) - 1,$
    where $\oplus$ denotes the XOR-operator. The expression $h(X)=4(X_1 \oplus X_2) - 1$ is a pure interaction. To verify this, consider the lower table of Figure \ref{fig: binary interaction}. The weighted average of each row and each column is zero, which formally means $\E(h\mid X_1)=0$ and $\E(h\mid X_2)=0$.\footnote{Note that the constant $-1$ is to mean-center the function, which is a necessary condition for our definition. Otherwise, the expressions $\E(h\mid X_1)$ and $\E(h\mid X_2)$ would be constant but not zero.} This means, we get the (unique) decomposition of $Y$ into the GAM $g^*(X)=4X_1+4X_2$ and the pure interaction $h^*(X)=4(X_1\oplus X_2) -1$. Using this decomposition, we get the same standalone and main effect dependence contributions as in Example \ref{Example: student redundancy} (\textcolor{v1gray}{$v(1)=9$}, \textcolor{v2gray}{$v(2)=9$}, \textcolor{cross-pred-purple}{cross-predictability of $2$}, and \textcolor{cov-purple}{covariance of $4$}). This illustrates that the standalone contributions and $\mathrm{Dep}\left(J, \bar{J} \right )$ are not affected by the pure interaction. The \textcolor{interaction-green}{interaction surplus $\Var(h^*)=3$} is simply added to these values causing $\Psi(1,2)$ and $v(1\cup 2)$ to be increased by $3$. %
\end{example}
\section{APPLYING THE DECOMPOSITION TO FEATURE IMPORTANCE}
\label{sec:decomposing-feature-importance}

Now, we show how to apply the DIP decomposition to explain Leave-One-Covariate-Out (LOCO) importance \citep{lei2018distribution,williamson2021nonparametric}, a popular feature importance technique.\\
The LOCO score is defined as the drop in predictive power when removing one variable from the full set of features and can be rewritten using Definition \ref{def:cooperative-impact} as%
\[LOCO_j := v\left(j \cup \bar{j}\right) - v\left(\bar{j}\right) = v(j) + \Psi\left(j,\bar{j}\right).\]
To explain the relevance of interactions and dependencies for a LOCO score, we can decompose the cooperative impact $\Psi\left(j,\bar{j}\right)$ using Theorem \ref{thm: decomposition for groups}.\\ %
Using the same trick, we can also explain the relevance of cooperation for Shapley effects \citep{song2016shapley}, also called SAGE values \citep{covert2020understanding} (Appendix \ref{app: shapley effects}). More generally, we can apply the method to any feature importance method based on predictive power comparisons of the form $v(S \cup T) - v(T)$.
\begin{figure*}[ht]
    \centering
    \begin{subfigure}[t]{0.47\textwidth}
         \centering
         \includegraphics[width=0.99\linewidth]{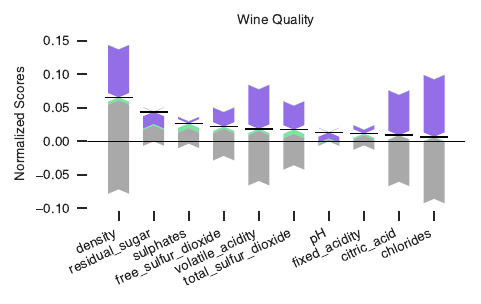}
    \end{subfigure}
    \hfill
    \begin{subfigure}[t]{0.47\textwidth}
        \centering
        \includegraphics[width=0.99\linewidth]{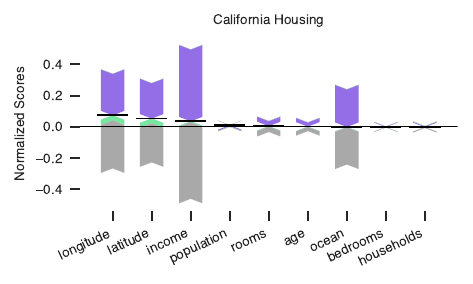}
     \end{subfigure}
    \caption{\textbf{Applications.} We decompose the LOCO scores on the wine quality dataset (left) and the California Housing dataset (right) into each feature's \textcolor{v2gray}{standalone contribution}, the \textcolor{interaction-green}{interaction surplus}, and the contribution of \textcolor{dependence-purple}{main effect dependencies}.}
    \label{fig:experiments-results}
\end{figure*}
\section{APPLICATIONS}\label{section: experiments}

As follows we apply the DIP decomposition to real-world data. More specifically, we compute and decompose the LOCO scores. We apply the method to two datasets, the wine quality dataset \citep{cortez2009modeling} and the California Housing dataset found in \citep{geron2022hands}. More detailed results, including results for SAGE, are reported in Appendix \ref{appendix:more-experimental-results}.
\paragraph{Implementation}{
To estimate the full model and to decompose it into pure interactions and main effects, we leverage the explainable boosting machine implemented in the \texttt{interpretML} package \citep{nori2019interpretml}. We compute the scores on test data as described in Appendix \ref{appendix:estimation-test-train}. We employ a $10$-fold cross-validation scheme. All scores are normalized by the variance of the target variable, thus indicating the proportion of the variance that is explained. We implemented the methods as a \texttt{python} package (\url{https://github.com/gcskoenig/dipd}); all code is publicly available (see Appendix \ref{appendix:estimation-implementation}, \url{https://github.com/gcskoenig/dipd-experiments}).}
}
\paragraph{Wine Quality}{For this dataset ($n=6496$, \cite{cortez2009modeling}), obtained from the UCI ML Repository \citep{Dua:2019} the goal is to predict \emph{wine quality} (a score between one and ten) using ten physicochemical characteristics such as \emph{citric acidity}, \emph{residual sugar}, and \emph{density}.
Suppose we use LOCO to gain insight into which variables are most relevant for predicting the target. 
The scores in Figure \ref{fig:experiments-results} (left, black horizontal lines) %
suggest that \emph{density} and \emph{residual sugar} are most relevant, but that \textit{citric acidity} is irrelevant. The scores are prone to misinterpretation.\\
First, one may erroneously infer that variables with large scores, like \textit{residual sugar}, also have large standalone predictive power. The DIP decomposition (Figure \ref{fig:experiments-results}, left) reveals that \textit{residual sugar} is relevant due to cooperation instead, and thus, its role in predicting the target can only be understood in combination with other features. A pairwise DIP decomposition further reveals a positive cooperative impact between \emph{residual sugar} and \emph{density} (Appendix \ref{appendix:pairwise-wine-quality}); an in-depth analysis shows that the features are positively correlated (adding sugar increases the density) but have opposing effects on wine quality that cancel out unless both features are observed (Appendix \ref{appendix:relation-density-sugar-quality}).\\
Second, one may erroneously conclude that features with small LOCO scores, like \textit{citric acidity}, contain little predictive information about $Y$. Instead, DIP reveals that \textit{citric acidity} is one of the most predictive standalone features but is considered irrelevant by LOCO due to its redundancy with the remaining features. A pairwise DIP decomposition reveals that \emph{citric acidity} shares most of its contribution with \textit{volatile acidity}, suggesting that they have similar roles for the target (Appendix \ref{appendix:pairwise-wine-quality}).}
\paragraph{California Housing}{The goal of the California Housing dataset ($n=20433$, \cite{geron2022hands}) is to predict the 1990 median house price of districts in California based on characteristics such as \textit{longitude}, \textit{latitude}, and \textit{ocean proximity}. %
The variables \emph{longitude} and \emph{latitude} have the highest LOCO scores, and one may erroneously conclude that the variables are individually important for the outcome. However, as DIP reveals, the variables' LOCO scores mostly stem from interactions (Figure \ref{fig:experiments-results}, right). As such, we need to consider both variables together to fully understand their role for $Y$. Furthermore, DIP reveals that seemingly irrelevant variables such as \textit{ocean proximity} and \textit{income} in fact are useful standalone predictors that receive low scores since they share their contributions with the remaining features.
}

\section{CONCLUSION}
Throughout the paper, we introduced DIP, a method to decompose the $\LL^2$-loss-based predictive power of two groups of features into their standalone predictive power as well as their cooperation via interactions and dependencies.\\
DIP can be used to explain the the outputs of commonly used global feature importance methods such as LOCO \citep{lei2018distribution,williamson2021nonparametric} or SAGE \citep{covert2020understanding}. More generally, DIP is applicable to any method based on predictive power comparisons involving two groups of features.\\
Thereby, DIP allows novel insight into the DGP: It reveals which features are individually relevant and which are due to interactions and dependencies, showing whether variables must be analyzed jointly to understand their role for the target. Furthermore, we can assess which variables share predictive contributions and thus have similar relationships with $Y$; insight that cannot be obtained by simply measuring the dependencies between features. 
Since these questions are of central interest to scientific inquiry, we are convinced that DIP has great potential to enable relevant discoveries.

\subsubsection*{Acknowledgements}

The authors thank the International Max Planck Research School for Intelligent Systems (IMPRS-IS) for supporting Eric Günther.
This work has been supported by the German Research Foundation through the Cluster of Excellence ``Machine Learning - New Perspectives for Science'' (EXC 2064/1 number 390727645). We thank Sebastian Bordt, Moritz Haas, and Karolin Frohnapfel for their feedback and Giles Hooker for a helpful email conversation about the uniqueness of the fANOVA decomposition.

\bibliography{main.bib}

\onecolumn
\appendix
\section{THEORY}
\label{appendix:theory}
\subsection{Proofs}
\subsubsection{Necessary Lemmata}\label{app: lemma subsection}
\begin{Lem}[\textbf{Equivalence of Orthogonality and Non-approximability}]\label{app: Lem: equivalence orthogonality and cond expectation}
    Let $f:\R^d\to \R$ be a %
    function on some probability space $\left(\R^d, P\right)$ and $J\subseteq \R$ a subsest of variables. Then the following are equivalent: 
    \begin{enumerate}
        \item $\E(f\cdot h)=0$ for all $h\in\LL^2\left(\R^J,P\right)$
        \item $\E(f\mid X_J)=0.$
    \end{enumerate} 
\end{Lem}
\begin{proof}\
    \begin{description}
        \item[$1.\Longrightarrow 2.$] Note that by setting $h\equiv 1$ we get $\E(f)=0$. We start by considering $h=\E(f\mid X_J)$ which gives $ \E(f\cdot \E(f\mid X_J))=0$. Since $f$ is centered we can replace the expected value of the product by a covariance and then use the law of total covariance to receive
        \begin{align*}
            0 &= \E(f\cdot \E(f\mid X_J)) =\Cov(f, \E(f\mid X_J)) \\&= \Cov(\E(f\mid X_J), \E(f\mid X_J)) + \underbrace{\E(\Cov(f, \E(f\mid X_J)\mid X_J))}_{=0} \\&= \Var(\E(f\mid X_J)),
        \end{align*}
        so $\E(f\mid X_J)$ is constant. Since its expected value yields $\E(\E(f\mid X_J))=\E(f) = 0$, the function $\E(f\mid X_J)$ must already be constant zero. 
    \item[$2.\Longrightarrow 1.$] Note that again, $\E(f)=\E(\E(f\mid X_J)) = 0.$ For an arbitrary $h\in\LL^2\left(\R^J,P\right)$, using the law of total covariance, we directly compute
    \begin{align*}
        \E(f\cdot h) = \Cov(f, h) = \Cov(\underbrace{\E(f\mid X_J)}_{=0}, \E(h\mid X_J)) + \underbrace{\E(\Cov(f, h\mid X_J))}_{=0} = 0.
    \end{align*}
    \end{description}
\end{proof}
\begin{Lem}[\textbf{Equivalence of Approximating the Data and Approximating a Better Predictor}]\label{app: Lem: tower property}
    Let $(X,Y)\sim P$ be a data generating process and consider two function classes $\mathcal{F}$ and $\mathcal{G}$ such that $\mathcal{G}\subseteq\mathcal{F}$. Let $f$  be the $\LL^2$-optimal predictor in the function class $\mathcal{F}$. Then, for a function  $g \in \mathcal{G}$ the following are equivalent:
    \begin{enumerate}
        \item The function $g$ is the $\LL^2$-optimal predictor for $Y$ within $\mathcal{G}$
        \item The function $g$ is the $\LL^2$-optimal approximation of $f$ within $\mathcal{G}$.
    \end{enumerate}
\end{Lem}
\begin{proof}
    Note that a predictor is $\LL^2$-optimal within a function class if and only if its residual is perpendicular to the function class \citep{luenberger1997optimization}. 
    \begin{description}
        \item[$1.\Longrightarrow 2.$] Let us assume that $g$ is optimal for $Y$. So, for an arbitrary $h\in \mathcal{G}$ we have
    \begin{align*}
        0&= \E((Y-g(X))\cdot h(X)) = \underbrace{\E((Y-f(X))\cdot h(X))}_{=0}+ \E((f(X)-g(X))\cdot h(X)) \\
        &= \E((f(X)-g(X))\cdot h(X)),
    \end{align*}
    where $\E((Y-f(X))\cdot h(X))$ vanishes because $f$ is the optimal predictor for $Y$ within $\mathcal{F}$ and $h$ is contained in $\mathcal{F}$. From that, we directly conclude that $g$ is $\LL^2$-optimal for $f(X)$.
    \item[$2.\Longrightarrow 1.$] If $g$ is optimal for $f(X)$ we again take an arbitrary $h\in \mathcal{G}$ and compute
    \begin{align*}
        0 &= \E((f(X)-g(X))\cdot h(X)) = \E((Y-g(X))\cdot h(X)) - \underbrace{\E((Y-f(X))\cdot h(X))}_{=0}\\
        &= \E((Y-g(X))\cdot h(X))
    \end{align*}
    and hence, $g$ is optimal for $Y$. 
            
    \end{description}
\end{proof}

\subsubsection{Proof of Theorem \ref{Thm: GAM residual gives pure interaction}}\label{app: proof equivalence pure interaction}
\begin{proof}[Proof of Theorem \ref{Thm: GAM residual gives pure interaction}]
     The function $g$ is the $\LL^2$-optimal approximation of $f$ within all GGAMs in $X_J$ and $X_{\bar{J}}$  if and only if $f-g$ is perpendicular to the class of all GGAMs in $X_J$ and $X_{\bar{J}}$ \citep{luenberger1997optimization}. This is equivalent of saying $f-g$ is perpendicular to all functions that only depend on $X_J$ and all functions that only depend on $X_{\bar{J}}$. The claim then follows from  Lemma \ref{app: Lem: equivalence orthogonality and cond expectation}.
\end{proof}

\subsubsection{Uniqueness of the GGAM Components}\label{app: uniqueness GAM}
Here, we prove that the GGAM components are unique up to a constant under some mild assumptions on the underlying distribution and the GGAM. Note that a stronger result than uniqueness up to a constant cannot be proven because it is always possible to add a constant to $g_J$ and subtract the same one from $g_{\bar{J}}$, no matter how strong our assumptions are. However, when decomposing the cooperative impact in Theorem \ref{thm: decomposition for groups}, we take variances and covariances and hence, these constants do not change the DIP decomposition.

\begin{thm}[\textbf{Uniqueness of the GGAM Components}]\label{Thm: uniqueness GAM}
	Let $g:\R^d\to \R$ be a measurable function on some probability space $\left(\R^d, P\right)$ that admits a decomposition
	$$g(x) = g_J(x_J) + g_{\bar{J}}(x_{\bar{J}}).$$ 
    Assume that one of the following three assumptions is satisfied:
    \begin{enumerate}
        \item The probability measure on $\R^d$ is discrete and $P(A_J\times A_{\bar{J}})>0$ for all values $A_J, A_{\bar{J}}$ that $X_J$ and $X_{\bar{J}}$ take 
        \item The probability measure on $\R^d$ is continuous with some strictly positive density $p>0$ and $g, g_J, g_{\bar{J}}$ only take finitely many values 
        \item The probability measure on $\R^d$ is continuous with some strictly positive density $p>0$ and $g, g_J, g_{\bar{J}}$ are continuous.
    \end{enumerate}
	Then, the components $g_J$ and $g_{\bar{J}}$ are unique almost surely up to a constant.
\end{thm}
\begin{proof}
	Given two decompositions $g = g_J + g_{\bar{J}} = \tilde{g}_J + \tilde{g}_{\bar{J}}$ we may subtract them to receive
	$$\left(g_J - \tilde{g}_{J}\right) + \left(g_{\bar{J}} - \tilde{g}_{\bar{J}}\right) = 0\quad a.s.$$
	For simplicity of the notation we write $f_J:=g_J - \tilde{g}_{J}$ and $f_{\bar{J}}:= -\left(g_{\bar{J}} - \tilde{g}_{\bar{J}}\right)$. So we know $f_J = f_{\bar{J}}$ a.s., from which we will derive that $f_J$ and $f_{\bar{J}}$ are already constant a.s. This is the same as saying $g_J$ and $\tilde{g}_J$ as well as $g_{\bar{J}}$ and $\tilde{g}_{\bar{J}}$ coincide up to a constant. \\
	Lets first consider the assumptions $1$ and $2$. In either of the two cases the two functions $f_J, f_{\bar{J}}$ only take finitely many values. For the sake of contradiction, assume one of the functions takes two distinct values $a_1, a_2$, each with probability greater than zero. W.l.o.g. we assume this is $f_J$. We denote the preimages of $a_1$ and $a_2$ with $A_1, A_2\subseteq \R^{J}$. Furthermore, let $b$ be a value that $f_{\bar{J}}$ takes with probability greater than zero and $B\subseteq \R^{\bar{J}}$ its preimage. W.l.o.g. we assume $b\neq a_1$, otherwise, we may switch $a_1$ and $a_2$. In case of assumption $1$ we know that $P(A_1\times B)>0$.\\
    Under assumption $2$ $P(A_1\times B)>0$ still holds true as we show in the following. Note that since the density $p$ is strictly positive, a subset of $\R^d$ has some positive probability if and only if it has some positive Lebesgue-measure. This holds true for $A_1\times \R^{\bar{J}}$ and $\R^J\times B$, because $P(f_J=a_1)>0$ and $P(f_{\bar{J}}=b)>0$. Hence, $A_1$ has some positive $J$-dimensional Lebesgue-measure and $B$ some positive $\bar{J}$-dimensional Lebesgue-measure. From that we conclude that $A_1\times B$ has some positive Lebesgue-measure too and hence $P(A_1\times B) > 0$ holds true under assumption 2 as well.\\
    But $f_J(A_1\times B) = a_1 \neq b = f_{\bar{J}}(A_1\times B)$ and therefore $P(f_J\neq f_{\bar{J}}) > 0$, which is a contradiction. Thus, $f_J$ and $f_{\bar{J}}$ are both constant a.s.\\
	Let us now consider assumption 3. The proof follows a very similar argumentation here. Let us again assume for the sake of contradiction that one of the two functions is not almost surely constant, w.l.o.g. we again assume this to be $f_J$, and let $a_1,a_2$ be two values that $f_J$ takes. As before, let $b$ be a value that $f_{\bar{J}}$ takes and assume w.l.o.g. that $b\neq a_1$. If $P(f_J=a_1)>0$ and $P(f_{\bar{J}}=b)>0$, the rest of the proof follows exactly as for assumption 2. Otherwise, pick two points $z\in f_J^{-1}(a_1)\subseteq\R^J, w\in f_{\bar{J}}^{-1}(b)\subseteq \R^{\bar{J}}$. Since the functions $f_J$ and $f_{\bar{J}}$ are continuous, points close to $z$ or $w$ map to points close to $a_1$ or $b$ respectively. More formally, there exist small balls $A_1:=
    B_{\eps_1}(z)\subseteq\R^J$ around $z$ and $B:=
    B_{\eps_2}(w)\subseteq\R^{\bar{J}}$ around $w$ for some $\eps_1,\eps_2>0$, such that $f_J(A_1) \cap f_{\bar{J}}(B)= \emptyset$. By construction, $A_1$ has some positive $J$-dimensional Lebesgue-measure and $B$ some positive $\left(\bar{J}\right)$-dimensional Lebesgue-measure. Hence, also $A_1\times B$ has some positive Lebesgue-measure, which again implies $P(A_1\times B) > 0$. But due to $f_J(A_1) \cap f_{\bar{J}}(B)= \emptyset$ we again receive $P(f_J\neq f_{\bar{J}}) > 0$, which is a contradiction. So, $f_J$ and $f_{\bar{J}}$ are constant a.s.
\end{proof}

\subsubsection{Proof of Theorem \ref{thm: decomposition for groups} and Proposition \ref{prop:the-only-coop-effects}}\label{app: proof main theorem}
Proposition \ref{prop:the-only-coop-effects} is a special case of Theorem \ref{thm: decomposition for groups} resulting from $X_J$ and $X_{\bar{J}}$ being independent and $h^*=0$. 
\begin{proof}[Proof of Theorem \ref{thm: decomposition for groups}]
    Note that $g^*$ is also the best $\LL^2$-approximation of $f^*$ due to Lemma \ref{app: Lem: tower property} and so, $h^*$ is the pure interaction of $f^*$. We begin by simply expanding $v(D)=\Var(f^*)$. We receive
    \begin{align}\label{Var exp}
        \Var(f^*) = &\Var (g^*_J + g^*_{\bar{J}} + h^*) \notag\\
        = &\Var(g^*_J) + \Var(g^*_{\bar{J}}) + \Var(h^*) \\
     &+ 2 \Cov(g^*_J, g^*_{\bar{J}}) + \underbrace{2\Cov(g^*_J, h^*)}_{=0} + \underbrace{2\Cov(g^*_{\bar{J}}, h^*)}_{=0}\notag,
    \end{align}
    where the last two summands vanish due to Theorem \ref{Thm: GAM residual gives pure interaction} and Lemma \ref{app: Lem: equivalence orthogonality and cond expectation}. We now consider the approximations based on only one of the two subsets of features. Remember that $\E(Y\mid X_J)= \E(f^*\mid X_J)$ as shown in Lemma \ref{app: Lem: tower property}. We compute
    \begin{align*}
        \E(Y\mid X_J) &= \E(f^*\mid X_J) = \E(g^*_J\mid X_J) + \E(g^*_{\bar{J}}\mid X_J) + \underbrace{\E(h^*\mid X_J) }_{=0} \\
        &= g^*_J + \E(g^*_{\bar{J}}\mid X_J), \\
        \E(Y\mid X_{\bar{J}}) &= g^*_{\bar{J}} + \E(g^*_J\mid X_{\bar{J}}),
    \end{align*}
    where we again used the fact that $h^*$ is a pure interaction. Computing variances leads to 
    \begin{align*}
        v(J) &= \Var(g^*_J) + \Var(\E(g^*_{\bar{J}}\mid X_J)) + 2\Cov(g^*_J, \E(g^*_{\bar{J}}\mid X_J))\\
        v\left(\bar{J}\right) &= \Var(g^*_{\bar{J}}) + \Var(\E(g^*_J\mid X_{\bar{J}})) + 2\Cov(g^*_{\bar{J}}, \E(g^*_J\mid X_{\bar{J}})).
    \end{align*}
    Using the law of total covariance we can simplify 
    \begin{align*}
        \Cov(g^*_J, \E(g^*_{\bar{J}}\mid X_J)) &= \Cov(\E(g^*_J\mid X_J), \E(g^*_{\bar{J}}\mid X_J)) \\
        &= \Cov(g^*_J, g^*_{\bar{J}}) - \underbrace{\E(\Cov(g^*_J, g^*_{\bar{J}}\mid X_J)) }_{=0}
    \end{align*}
    and analogously, 
    \begin{align*}
        \Cov(g^*_{\bar{J}}, \E(g^*_J\mid X_{\bar{J}})) = \Cov(g^*_J, g^*_{\bar{J}}),
    \end{align*}
    leading to 
    \begin{align}\label{Var uni exp 1}
        v(J) &= \Var(g^*_J) + \Var(\E(g^*_{\bar{J}}\mid X_J)) + 2\Cov(g^*_J, g^*_{\bar{J}})\\
        v\left(\bar{J}\right) &= \Var(g^*_{\bar{J}}) + \Var(\E(g^*_J\mid X_{\bar{J}})) + 2\Cov(g^*_{\bar{J}}, g^*_J). \label{Var uni exp 2}
    \end{align}
    Putting the equations \eqref{Var exp}, \eqref{Var uni exp 1} and \eqref{Var uni exp 2} together proves the claim.
\end{proof}
\subsection{Examples}
\subsubsection{Example \ref{example:cooperative-forces-cancel-out}: Linear Function on Normally Distributed Data}\label{app:example:linear}
We consider 
$$Y=X_1 + X_2 + cX_1X_2$$
where $X\sim \N\left (0, \left( \begin{array}{cc}
		1 & \beta \\
		\beta & 1
	\end{array}\right) 
	\right )$ for arbitrary $c\in\R, \beta\in[0,1)$. In the following, we will derive the value functions, the functional decomposition and the DIP decomposition in a purely theoretical way. 
\paragraph{Mathematical Basics and Precomputations}
For all these computations we need the generalized Isserlis theorem, formalized in \cite{withers1985moments}.
\begin{thm*}[\textbf{Withers, 1985}]
    If $(X_1, ..., X_d)$ is a zero-mean multivariate normal random vector and $A=\{\alpha_1,...,\alpha_l\}$ is a subset of (not necessarily distinct) indices between $1$ and $d$, we have
    \begin{align*}
        \E(X_{\alpha_1}\cdots X_{\alpha_l}) = \sum_{p\in P_A^2}\prod_{\{i,j\}\in p} \E(X_iX_j).
    \end{align*}
    Here, $P_A^2$ is the set of all possible partitions of the set $A$ into pairs. The product then goes over all these pairs in a particular partition $p$. 
\end{thm*}
This means in particular that the expression vanishes if $l$ is odd, because in this case there are zero possible partitions of $\{\alpha_1,...,\alpha_l\}$ into pairs.\\
With the aid of this theorem we compute the following expressions that will be used afterwards. Note that $\E(X_1^2)=\E(X_2^2)=1$ and $\E(X_1X_2)=\beta$.
\begin{align*}
		\Var(X_1X_2) &= \E\left( X_1^2X_2^2\right) - \E\left( X_1X_2\right)^2 \\
							 &= \E\left (X_1^2\right )\E\left (X_2^2\right ) + \E\left (X_1X_2\right )^2+ \E\left (X_1X_2\right )^2 - \E(X_1X_2)^2 \\
							 &= \E\left (X_1^2\right )\E\left (X_2^2\right ) + \E(X_1X_2)^2 \\
							 &= 1+\beta^2,\\
		\Var\left (X_1^2\right ) &= \E\left (X_1^4\right ) - \E\left (X_1^2\right )^2
											=3\E\left (X_1^2\right )^2- \E\left (X_1^2\right )^2
											=2\E\left (X_1^2\right )^2  \\
											&=2,\\
		\Cov\left (X_1,X_1X_2\right ) &= \E\left( X_1^2X_2\right) - \E(X_1)\E(X_1X_2) =0,\\
		\Cov(X_2,X_1X_2)&=0,\\
		\Cov\left (X_1,X_1^2\right ) &= \E\left (X_1^3\right ) - \E\left (X_1^2\right )\E\left (X_1\right ) = 0.\\
        \Cov\left (X_2,X_2^2\right ) &= 0\\
        \Cov\left(X_1X_2, X_1^2\right) &= \E\left(X_1^3X_2\right) - \E(X_1X_2)\E\left(X_1^2\right) = 3\E\left(X_1^2\right)\E(X_1X_2) - \E(X_1X_2)\E\left(X_1^2\right) \\
            &= 3\beta - \beta = 2\beta\\
        \Cov\left(X_1X_2, X_2^2\right) &= 2\beta\\
        \Cov\left(X_1^2, X_2^2\right) &= \E\left(X_1^2X_2^2\right) - \E\left(X_1^2\right)\E\left(X_2^2\right) \\
        &= \E\left(X_1^2\right)\E\left(X_2^2\right) +    2\E(X_1X_2)^2 - \E\left(X_1^2\right)\E\left(X_2^2\right) = 2\E(X_1X_2)^2 \\
        &= 2\beta^2.
	\end{align*}
\paragraph{Computation of the Value Functions}
With these expressions, we can now compute
	\begin{align*}
		v(1\cup 2) &= \Var(Y) = \Var(X_1+
  X_2+cX_1X_2) \\
  &= \Var(X_1)  + \Var(X_2) + c^2\Var(X_1X_2) 
				+ 2\Cov(X_1,X_2) + 2c\underbrace{\Cov(X_1, X_1X_2)}_{=0} + 2c\underbrace{\Cov(X_2,X_1X_2)}_{=0}\\
				&= 2+ 2\beta +c^2(1+\beta^2).
	\end{align*} 
Note that $\E(X_2\mid X_1)=\beta X_1$, which is known from the conditional distribution of a multivariate normal. With this identity in mind, we compute
\begin{align*}
	v(1) &= \Var(\E(Y\mid X_1)) = \Var(\E(X_1+X_2+cX_1X_2\mid X_1))\\
 &= \Var(X_1 + \E(X_2\mid X_1) + cX_1\E(X_2\mid X_1))
 = \Var\left (X_1 +\beta X_1+c\beta X_1^2\right )\\
								&=(1+\beta)^2\Var(X_1) + c^2\beta^2 \Var\left (X_1^2\right ) + 2(1+\beta)c\beta\underbrace{\Cov\left (X_1,X_1^2\right )}_{=0}\\
								&= (1+\beta)^2 + 2c^2\beta^2.
\end{align*}
Analogously, 
$$v(2) = (1+\beta)^2 + 2c^2\beta^2.$$
Plugging in $c=\beta=0$ and $c=\sqrt{6}, \beta = 0.5$ from our two DGPs in Example \ref{example:cooperative-forces-cancel-out} and dividing by $\Var(Y)=2+2\beta+c^2(1+\beta^2)$ we receive 
\begin{align*}
    \bar{v}(1\cup 2) = 1, \quad \bar{v}(1)=\frac{1}{2}, \quad \bar{v}(2)=\frac{1}{2}
\end{align*}
as claimed. 
\paragraph{Functional Decomposition}
For the components of the cooperative impact, we first need the functions $g^*_1, g^*_2$ and $h^*$. Note that $Y=f^*(X)$. The component functions are given by
\begin{align*}
    g^*_1(X_1) &= X_1 + \frac{c\beta}{1+\beta^2}X_1^2 - \frac{c\beta(1-\beta^2)}{2(1+\beta^2)}\\
    g^*_2(X_2) &= X_2 + \frac{c\beta}{1+\beta^2}X_2^2 - \frac{c\beta(1-\beta^2)}{2(1+\beta^2)}\\
    h^*(X) &= cX_1X_2 - \frac{c\beta}{1+\beta^2}X_1^2 - \frac{c\beta}{1+\beta^2}X_2^2 + \frac{c\beta(1-\beta^2)}{1+\beta^2}.
\end{align*}
To prove this, we verify $\E(h^*\mid X_1)=\E(h^*\mid X_2)=0$, which is equivalent to $g^*$ being the $\LL^2$-optimal GAM as we showed in Theorem \ref{Thm: GAM residual gives pure interaction}. We first note that 
\begin{align*}
    \E\left(X_2^2\mid X_1\right) = \Var(X_2\mid X_1) + \E(X_2\mid X_1)^2 = 1-\beta^2 + \beta^2X_1^2
\end{align*}
where $\Var(X_2\mid X_1)$ and $\E(X_2\mid X_1)$ are known from the conditional distributions of the multivariate normal distribution. We then compute
\begin{align*}
    \E(h^*\mid X_1) &= cX_1\E(X_2\mid X_1) - \frac{c\beta}{1+\beta^2}X_1^2 - \frac{c\beta}{1+\beta^2}\E\left(X_2^2\mid X_1\right) + \frac{c\beta(1-\beta^2)}{1+\beta^2}\\
     &= c\beta X_1^2 - \frac{c\beta}{1+\beta^2}X_1^2 - \frac{c\beta(1-\beta^2)}{1+\beta^2} - \frac{c\beta^3}{1+\beta^2}X_1^2 + \frac{c\beta(1-\beta^2)}{1+\beta^2}\\ 
     &= c\beta X_1^2 - \frac{c\beta}{1+\beta^2}X_1^2 - \frac{c\beta^3}{1+\beta^2}X_1^2 \\
      &= \frac{c\beta(1+\beta^2) - c\beta}{1+\beta^2}X_1^2 - \frac{c\beta^3}{1+\beta^2}X_1^2\\
      &=\frac{c\beta^3}{1+\beta^2}X_1^2 - \frac{c\beta^3}{1+\beta^2}X_1^2\\
      &=0.
\end{align*}
Analogously, one can compute $\E(h^*\mid X_2)=0$. Hence, $Y-h^*(X)$ is the $\LL^2$-optimal GAM for $Y$. The two components $g^*_1$ and $g^*_2$ are unique up to a constant as we know from Theorem \ref{Thm: uniqueness GAM}. We simply split the additive constant equally between the two. For the computation of the variance they can simply be dropped. 
\paragraph{Decomposition of the Cooperative Impact}
Again taking use of the quantities we computed with the generalized Isserlis theorem, we can now determine the components of the cooperative impact. For the cross-predictability we get
\begin{align*}
    \Var\left (\E\left (g^*_1\mid X_2\right )\right ) 
     &= \Var\left( \E(X_1\mid X_2) + \frac{c\beta}{1+\beta^2} \E\left(X_1^2\mid X_2\right)\right)\\
      &= \Var \left( \beta X_2 + \frac{c\beta}{1+\beta^2}\left(1+\beta^2 + \beta^2X_2^2\right)\right)\\
      &= \Var \left( \beta X_2 + \frac{c\beta}{1+\beta^2}\left(\beta^2X_2^2\right)\right)\\
      &= \beta^2\Var(X_2) + \left(\frac{c\beta^3}{1+\beta^2}\right)^2\Var\left(X_2^2\right) + 2\frac{c\beta^4}{1+\beta^2}\underbrace{\Cov\left(X_2, X_2^2\right)}_{=0}\\
      &= \beta^2 + \frac{2c^2\beta^6}{(1+\beta^2)^2}
\end{align*}
and in the same manner
\begin{align*}
    \Var\left (\E\left (g^*_2\mid X_1\right )\right )  = \beta^2 + \frac{2c^2\beta^6}{(1+\beta^2)^2}.
\end{align*}
Summed up we have
\begin{align*}
    \Var\left (\E\left (g^*_1\mid X_2\right )\right )  + \Var\left (\E\left (g^*_2\mid X_1\right )\right )  = 2\beta^2 + \frac{4c^2\beta^6}{(1+\beta^2)^2}.
\end{align*}
For the covariance we compute
\begin{align*}
    \Cov(g^*_1, g^*_2) &= \Cov\left(X_1 + \frac{c\beta}{1+\beta^2}X_1^2, X_2 + \frac{c\beta}{1+\beta^2}X_2^2\right)\\
    &= \Cov(X_1, X_2) + \frac{c\beta}{1+\beta^2}\underbrace{\Cov\left(X_1, X_2^2\right)}_{=0}\\ 
    &\quad+ \frac{c\beta}{1+\beta^2} \underbrace{\Cov\left(X_2, X_1^2\right)}_{=0} + \left(\frac{c\beta}{1+\beta^2}\right)^2\Cov\left(X_1^2, X_2^2\right)\\
    &=\beta + \frac{2c^2\beta^4}{(1+\beta^2)^2}.
\end{align*}
Eventually, for the interaction surplus we have 
\begin{align*}
    \Var(h^*) &= \Var\left(cX_1X_2 - \frac{c\beta}{1+\beta^2}X_1^2 - \frac{c\beta}{1+\beta^2}X_2^2\right)\\
    &= c^2\Var(X_1X_2) + \frac{c^2\beta^2}{(1+\beta^2)^2}\Var\left(X_1^2\right) + \frac{c^2\beta^2}{(1+\beta^2)^2}\Var\left(X_2^2\right) \\
    &\quad- 2\frac{c^2\beta}{1+\beta^2}\Cov\left(X_1X_2, X_1^2\right) - 2\frac{c^2\beta}{1+\beta^2}\Cov\left(X_1X_2, X_2^2\right) \\
    &\quad +2\frac{c^2\beta^2}{(1+\beta^2)^2}\Cov\left(X_1^2, X_2^2\right)\\
    &= c^2\left(1+\beta^2\right) + 4\frac{c^2\beta^2}{(1+\beta^2)^2} -8\frac{c^2\beta^2}{1+\beta^2} + 4 \frac{c^2\beta^4}{(1+\beta^2)^2}\\
    &= c^2\left(1+\beta^2\right) + 4\frac{c^2\beta^2\left(1+\beta^2\right)}{(1+\beta^2)^2} - 8\frac{c^2\beta^2}{1+\beta^2}\\
    &= c^2\left(1+\beta^2\right) - 4\frac{c^2\beta^2}{1+\beta^2}.
\end{align*}
Plugging in the values for $c$ and $\beta$ delivers the values of the DIP decomposition. Those can be normalized by dividing by $\Var(Y)=2+2\beta+c^2(1+\beta^2)$ to receive the values presented in the paper.
\subsubsection{Main Effect Dependencies Cannot Be Derived by Only Value Functions and Feature Correlation}\label{app:examples:correlation vs main effect dependencies}
In the following, we give an example of three two-dimensional DGPs, where $X$ follows the same distribution in each of the three and the value functions for all subsets of features coincide. However, we get different DIP decompositions for each DGP, illustrating that we cannot deduce the main effect dependencies from just the dependencies of the features, even if all the value functions are known.

Consider the following DGPs:
\begin{description}
    \item[DGP 1:] 
    \begin{align*}
        X &\sim \N\left(0, \left(\begin{array}{cc}1 & 0.5 \\ 0.5 & 1\end{array}\right)\right)\\
        Y &= -4.3X_1 -0.9X_2 - 3.9X_1^2 + 3.0X_2^2
    \end{align*}
    \item[DGP 2:]
    \begin{align*}
        X &\sim \N\left(0, \left(\begin{array}{cc}1 & 0.5 \\ 0.5 & 1\end{array}\right)\right)\\
        Y &= -1.3X_1 - 4.7X_2 + 3.6X_1^2 - 3.0X_2^2 + 4.7X_1X_2
    \end{align*}
    \item[DGP 3:]
    \begin{align*}
        X &\sim \N\left(0, \left(\begin{array}{cc}1 & 0.5 \\ 0.5 & 1\end{array}\right)\right)\\
        Y &= 10.9X_1 + 2.4X_2 -5.1X_1^2 - 5.3X_2^2 + 11.3X_1X_2.
    \end{align*}    
\end{description}
In each of the three cases, we get the (empirical and rounded) normalized value functions 
$$\bar{v}(1)=0.7,\quad \bar{v}(2)=0.3,\quad \bar{v}(1\cup 2) = 1.$$
Despite the coinciding value functions and the same correlation, we get three different DIP decompositions as the following plot shows.
\begin{figure}[H]
    \centering
    \includegraphics[width=0.8\linewidth]{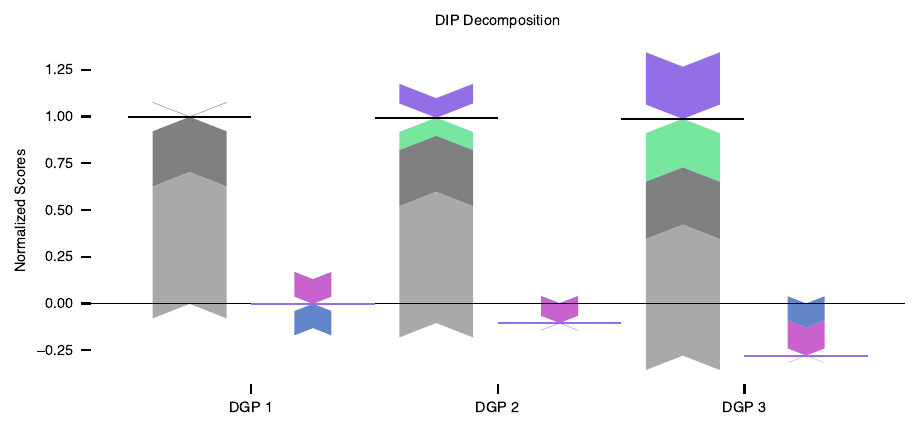}
    \caption{For each example, we show a forceplot visualizing the DIP decomposition into standalone contributions \textcolor{v1gray}{($v(1)$} and \textcolor{v2gray}{$v(2)$}), \textcolor{dependence-purple}{main effect dependencies ($\mathrm{Dep}(1,2)$)} and \textcolor{interaction-green}{interaction suplus}, where the direction of each bar (upward or downward) represents the sign. They sum up to $v(1,2)$ (black horizontal line). The slim bars (right) show the decomposition of \textcolor{dependence-purple}{$\mathrm{Dep}(1,2)$} (purple horizontal line) into \textcolor{cov-purple}{covariance} and \textcolor{cross-pred-purple}{cross-predictability}.}
    \label{fig: 3DGPs same correlation different main effect dependencies}
\end{figure}
\subsubsection{Vanishing Main Effect Dependencies Despite Variable Dependence}\label{app:examples: zero main order dependencies}
In the following, we show an example of a two-dimensional DGP, where the two features are strongly correlated, but their main effect dependencies vanish. 

Let $Z_0, Z_1, Z_2\sim \mathrm{Unif}(0,...,9)$ be independently distributed. We define 
\begin{align*}
X_1 &= 10Z_0 + Z_1\\
X_2 &= 10Z_0 + Z_2\\
Y &= (X_1\bmod 10) + (X_2\bmod 10)  = Z_1 + Z_2,
\end{align*}
that is, $Z_0$ defines the first digit and $Z_1, Z_2$ define the second digit of $X_1$ and $X_2$. The target variable $Y$ is just the sum of their last digits.\\
Clearly, $X_1$ and $X_2$ are strongly correlated due to $Z_0$. However, we have 
\begin{align*}
g_1(X_1) &= X_1 \bmod 10 = Z_1\\
g_2(X_2) &= X_2 \bmod 10 = Z_2
\end{align*}
and therefore $\E(g_1\mid X_2) = \E(Z_1\mid 10Z_0+Z_2)= \E(Z_1)$, $\E(g_2\mid X_1)=\E(Z_2\mid 10Z_0+Z_1)= \E(Z_2)$. Hence, the cross-predictability
$$\Var(\E(g_1\mid X_2)) + \Var(\E(g_2\mid X_2)) = 0$$
as well as the main effect covariance 
$$2\Cov(g_1, g_2) = \Cov(Z_1, Z_2) = 0$$
vanish.

Here, we see an example where all of the dependencies between $X_1$ and $X_2$ stem from their shared first digit $Z_0$, which is independent from the target $Y$. All of their information about $Y$ is contained in their last digits $Z_1,Z_2$, which are independent. This means that although our variables have a strong dependencies, they do not share information about the target, so their main effects dependencies vanish. This once more illustrates that main order dependencies cannot be determined by just the correlation, rather, they measure the dependencies of only those parts of our features that are relevant for predicting. 
\subsubsection{Binary Examples \ref{Example: student redundancy} - \ref{example: binary interaction}}\label{app:examples:student examples}
\begin{figure}[H]
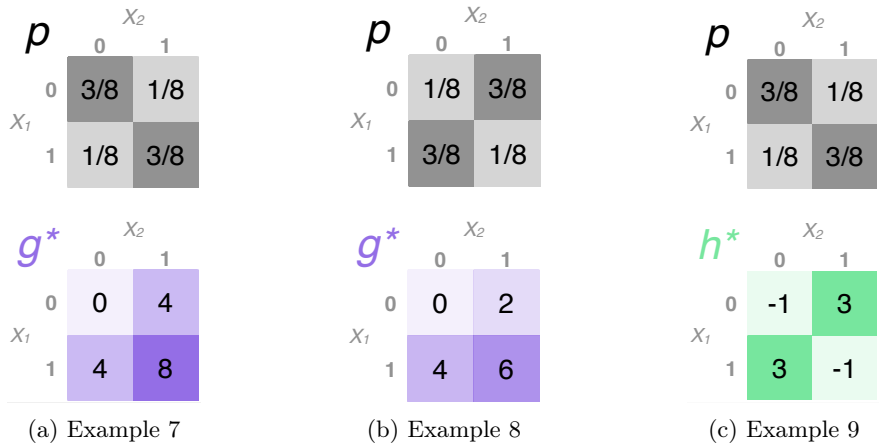

  \centering
    \begin{subfigure}{0.15\linewidth}
      \centering
        \includegraphics[width=\linewidth]{img/illustration_terms/DGP1.pdf}
        \caption{Example \ref{Example: student redundancy}}
    \end{subfigure}
    \hspace{0.1\linewidth}
    \begin{subfigure}{0.15\linewidth}
      \centering
        \includegraphics[width=\linewidth]{img/illustration_terms/DGP2.pdf}
        \caption{Example \ref{example: student enhancement}}
    \end{subfigure}
    \hspace{0.1\linewidth}
    \begin{subfigure}{0.15\linewidth}
      \centering
        \includegraphics[width=\linewidth]{img/illustration_terms/DGP3.pdf}
        \caption{Example \ref{example: binary interaction}}
    \end{subfigure}
    \caption{Visualizing the data generating process of the three illustrative student examples.}
    \label{fig:binary-examples-DGPs}
\end{figure}
\paragraph{Example \ref{Example: student redundancy}}
Since $Y=4X_1+4X_2$ can exactly been determined by a GAM, we have $f^*(X)=g^*(X)=4X_1+4X_2$ and $h^*=0$. Note that 
\begin{align*}
    \E(X_2\mid X_1=1) &= \frac{3}{4},\\
    \E(X_2\mid X_1=0) &= \frac{1}{4}.
\end{align*}
We can also write this as $\E(X_2\mid X_1)=0.5X_1+0.25$. In the same way, $\E(X_1\mid X_2)=0.5X_2+0.25$.\\
The best univariate predictors are given by
\begin{align*}
    \E(Y\mid X_1) &= \E(4X_1 + 4X_2\mid X_1) = 4X_1 + 4(0.5X_1+0.25) = 6X_1 + 1,\\
    \E(Y\mid X_2) &= 6X_2 + 1.
\end{align*}
The variance of a $\mathrm{Ber}(p)$-distributed variable is given by $p(1-p)$, so $\Var(X_i)=0.25$ in our case. This leads to
\begin{align*}
    v(1) &= \Var(\E(Y\mid X_1)) = \Var(6X_1+1) = 36\Var(X_1) = 9\\
    v(2) &= 9\\
    v(1\cup 2) &= \Var(Y) = \E\left(Y^2\right)-\E(Y)^2 = \frac{3}{8}\cdot 8^2 + 2\cdot \frac{1}{8}\cdot 4^2 - 4^2= 12. 
\end{align*}
This leads to main effect dependencies of $6$ because there is no interaction surplus. For main effect covariance we get
\begin{align*}
    2\Cov(4X_1,4X_2) = 32\Cov(X_1,X_2) = 32(\E(X_1X_2)-\E(X_1)\E(X_2))= 32\cdot\left(\frac{3}{8} - \frac{1}{4}\right) = 32\cdot \frac{1}{8} = 4,
\end{align*}
leading to a cross-predictability of $6-4=2$.
\paragraph{Example \ref{example: student enhancement}}
Again, $f^*(X)=g^*(X)=4X_1+2X_2$ and $h^*=0$. This time,
\begin{align*}
    \E(X_2\mid X_1=1) &= \frac{1}{4},\\
    \E(X_2\mid X_1=0) &= \frac{3}{4},
\end{align*}
which means $\E(X_2\mid X_1)=-0.5X_1 +0.25$ and $\E(X_1\mid X_2)=-0.5X_2+0.25$. For the univariate predictors we compute
\begin{align*}
    \E(Y\mid X_1) &= \E(4X_1+2X_2\mid X_1) = 4X_1 + 2(-0.5X_1+0.25) =3X_1 + 0.5.\\
    \E(Y\mid X_2) &= \E(4X_1 + 2X_2\mid X_2)  = 4\cdot (-0.5X_2 + 0.25) + 2X_2 = 1,
\end{align*}
leading to 
\begin{align*}
    v(1) &= \Var(\E(Y\mid X_1)) = \Var(3X_1+0.5) = 9\Var(X_1) = 2.25\\
    v(2) &= \Var(1) = 0\\
    v(1\cup 2) &= \Var(Y) = \E\left(Y^2\right)-\E(Y)^2 = \frac{1}{8}\cdot 8^2 + \frac{3}{8}\cdot 4^2 + \frac{3}{8}\cdot 2^2 - 3^2= 3. 
\end{align*}
This implies $\mathrm{Dep}(1,2) = -0.75$ because again, the interaction surplus vanishes. The main effect covariance is given by
\begin{align*}
    2\Cov(4X_1, 2X_2) = 16\Cov(X_1,X_2) = 16(\E(X_1X_2) - \E(X_1)\E(X_2)) = 16\cdot \left(\frac{1}{8} - \frac{1}{4}\right) = -2,
\end{align*}
so the cross-predictability is $-0.75+2=1.25$.
\paragraph{Example \ref{example: binary interaction}}
Note that the distribution of $X$ is the same as in Example \ref{Example: student redundancy}. We first prove that the decomposition $Y=g^*(X)+h^*(X)$ for  $g^*(X)=4X_1+4X_2$ and $h^*(X)=4(X_1\oplus X_2) - 1$ is indeed the unique functional decomposition into a GAM and a pure interaction. According to Theorem \ref{Thm: GAM residual gives pure interaction} it is sufficient to prove $\E(h^*\mid X_1)=\E(h^*\mid X_2)=0$. We get
\begin{align*}
    \E(4(X_1\oplus X_2)-1\mid X_1=1) &= \underbrace{P(X_2=1\mid X_1=1)}_{=\frac{3}{4}} \cdot (0 - 1) + \underbrace{P(X_2=0\mid X_1=1)}_{=\frac{1}{4}}\cdot(4 - 1) = 0  \\
    \E(4(X_1\oplus X_2)-1\mid X_1=0) &= \underbrace{P(X_2=1\mid X_1=0)}_{=\frac{1}{4}} \cdot (4 - 1) + \underbrace{P(X_2=0\mid X_1=0)}_{=\frac{3}{4}}\cdot(0 - 1) = 0.
\end{align*}
So, it holds $\E(h^*\mid X_1)=0$ and due to the symmetry also $\E(h^*\mid X_2)=0$. In particular, $\E(h^*) = \E(\E(h^*\mid X_1))=0$, so $h^*$ is centered. The univariate predictor yields
\begin{align*}
    \E(Y\mid X_1) = \E(g^*\mid X_1) + \underbrace{\E(h^*\mid X_1)}_{=0} = \E(4X_1+4X_2\mid X_1),
\end{align*}
likewise for $X_2$. This means, the univariate predictors are the same as in Example \ref{Example: student redundancy}, from which we may conclude that also $v(1)=v(2)=9$. Since $g^*$ as well coincides with the one in Example \ref{Example: student redundancy}, we can also conclude that $\mathrm{Dep}(1,2)=6$, the cross-predictability equals 2 and the main effects covariance equals 4.  For the interaction surplus we compute
$$\Var(h^*)= \E\left(h^2\right)-\underbrace{\E(h)^2}_{=0} = \frac{3}{4}\cdot (-1)^2 + \frac{1}{4}\cdot 3^2 = 3. $$
From this, we right away conclude 
$$v(1\cup 2) = v(1)+v(2) + \Var(h^*) - \mathrm{Dep}(1,2) = 15.$$

\subsection{Decomposition of SAGE}\label{app: shapley effects}
The SAGE value \citep{covert2020understanding}, also known as Shapley effect \citep{song2016shapley} of a feature $j$ is defined as
\begin{align}\label{eq: app: shap}
    \Phi_j := \sum_{S\subseteq D\setminus j} c_S\cdot (v(S\cup j) - v(S)),
\end{align}
where the weights $c_S$ are given by
$$c_S= \frac{(d-|S|-1)!|S|!}{d!}.$$
Note that $\sum_{S\subseteq D\setminus j}c_S = 1$. In the following, we write $\mathrm{Int}(j, S)$ for the interaction surplus and $\mathrm{Dep}(j,S)$ for the main order dependencies between the feature groups $j$ and $S$. By decomposing every summand of \eqref{eq: app: shap} we receive
\begin{align*}
    \Phi_j :&= \sum_{S\subseteq D\setminus j} c_S\cdot (v(S\cup j) - v(S)) = \sum_{S\subseteq D\setminus j} c_S\cdot (v(j) + \mathrm{Int}(j, S) - \mathrm{Dep}(j,S)) \\
    &= v(j) + \sum_{S\subseteq D\setminus j} c_S\cdot \mathrm{Int}(j, S) - \sum_{S\subseteq D\setminus j} c_S\cdot \mathrm{Dep}(j,S),
\end{align*}
yielding a decomposition of $\Phi_j$ into its standalone contribution, the contribution stemming from interactions and the one stemming from dependencies. Here, instead of only computing the interaction surplus and the main order dependencies between $j$ and $\bar{j}$, we do so for $j$ and every subset $S\subseteq D\setminus j$ and compute a weighted average, just as in the formula for the Shapley effect.\\
One disadvantage of the Shapley effect over LOCO is the exponential runtime in the number of features since we need to refit a model for every $S\subseteq D\setminus j$ or at least use a sufficiently large number of subsets for an approximation.

\newpage
\section{DETAILS ON ESTIMATION AND IMPLEMENTATION}
\label{appendix:estimation}
\subsection{Estimation}
\label{appendix:estimation-test-train}

In Theorem \ref{thm: decomposition for groups} we derived a decomposition of the cooperative impact, assuming the optimal models $f^*$ and $g^*$ and the true data distribution $(X,Y)\sim P$ to be known. In practice, the optimal models are not available; Instead we fit and evaluate ML models using some dataset $\left ( (x,y)^{(1)}, \dots, (x, y)^{(n)} \right )$ with indices $\mathcal{I} = \{1, \dots, n\}$.\\
To avoid bias due to overfitting, we reformulate the quantities such that they can be estimated using test set performance. We start by reformulating interaction surplus and main effect dependencies in terms of value functions.

\paragraph{Interaction Surplus And Main Effect Dependencies In Terms of Value Functions}

First, we recall that we defined the interaction surplus for optimal model $f^*$ and optimal GGAM $g^*$ as $\Var(h^*)$ where $h^* := f^* - g^*$. In Lemma \ref{lemma:estimating-pure-interaction-from-data}, we show that we can equivalently define the interaction surplus, in the following denoted as $\mathrm{Int} \left ( J, \bar{J} \right )$, as the difference in value function between the full model $f^*$ and the GGAM $g^*$.
\begin{Lem}
\label{lemma:estimating-pure-interaction-from-data}
Let $(X,Y)\sim P$ be a DGP, $J\subseteq D$ a subset of features, $f^*$ the $\LL^2$-optimal predictor and $g^*=g^*_J+g^*_{\bar{J}}$ the $\LL^2$-optimal GGAM in $X_J$ and $X_{\bar{J}}$. We call $h^*=f^*-g^*$. Then, the interaction surplus $\mathrm{Int}\left(J,\bar{J}\right) := Var(h^*)$ is given by
$$\mathrm{Int}\left(J,\bar{J}\right)= v_{f^*}\left(J \cup \bar{J}\right) - v_{g^*}\left(J \cup \bar{J}\right).$$
\end{Lem}

\begin{proof}
    We compute
    \begin{align*}
        v_{f^*}\left(J \cup \bar{J}\right) - v_{g^*}\left(J \cup \bar{J}\right) &= \E\left((Y-g^*)^2\right) - \E\left((Y-f^*)^2\right)\\
        &= \E\left((Y-f^*+h^*)^2\right) - \E\left((Y-f^*)^2\right)\\
        &= \E\left((Y-f^*)^2\right) + 2\underbrace{\E\left((Y-f^*)\cdot h^*\right)}_{=0} + \E\left((h^*)^2\right) - \E\left((Y-f^*)^2\right)\\
        &=\Var(h^*).
    \end{align*}
    Here, $\E\left((Y-f^*)\cdot h^*\right)$ vanishes because $f^*$ is the optimal predictor, and its residual is hence perpendicular to every function in $\LL^2\left(\R^d,P\right)$ \citep{luenberger1997optimization}. Furthermore, $h^*$ is mean-centered due to Theorem \ref{Thm: GAM residual gives pure interaction} and thus $\E\left(h^2\right)=\Var\left(h^2\right)$.
\end{proof}

\paragraph{Estimaton of $\mathrm{Dep}$ and $\mathrm{Int}$}

As a consequence of Lemma \ref{lemma:estimating-pure-interaction-from-data}, we can estimate the interaction surplus for some $\hat{f}$, $\hat{g}$ and observation indices $\mathcal{I}$ as the difference in estimated value functions. Therefore, we first recapitulate that the definition the value function for the $\LL^2$-loss and a model $f$ is given by
$$v_{f, \LL^2}(S) := E\left((f_\emptyset - Y)^2\right) - E\left((f_S(X_S) - Y)^2\right),$$
which can be estimated in terms of test data with indices $\mathcal{I}_{te}$ using the empirical risk by
\begin{align*}
    \hat{v}^{\mathcal{I}_{te}}_{f, \LL^2}(S) = \frac{1}{|\mathcal{I}_{te}|} \left (\sum_{i \in \mathcal{I}_{te}} \left (f_\emptyset - y^{(i)} \right )^2 - \sum_{i \in \mathcal{I}_{te}} \left (f_S\left(x^{(i)}\right) - y^{(i)} \right )^2 \right ).
\end{align*}
We recall that $f_S$ can be obtained using refitting or via conditional integration. The term $f_\emptyset$ is the best constant approximation of $f$, which is $\E(f)$. For an optimal predictor, this is the same as $\E(Y)$. Both of these terms can be approximated by the empirical mean. To avoid biased results, all models must be fit on training data with indices $\mathcal{I}_{tr}$ where $\mathcal{I}_{tr} \cap \mathcal{I}_{te} = \emptyset$.\\
Based on the empirical value function, we can estimate the interaction surplus for some test set $\mathcal{I}_{te}$ as
\begin{align*}
    \widehat{\mathrm{Int}}^{\mathcal{I}_{te}}_{\hat{f}, \hat{g}} \left( J,\bar{J} \right) := \hat{v}_{\hat{f}, \LL^2}^{\mathcal{I}_{te}}\left(J\cup \bar{J}\right) - \hat{v}_{\hat{g}, \LL^2}^{\mathcal{I}_{te}} \left(J\cup\bar{J}\right).
\end{align*}
To estimate $\mathrm{Dep}\left (J, \bar{J} \right )$, we recall that $\Psi\left(J, \bar{J}\right) 
        =\textcolor{interaction-green}{\Var(h^*)} - \textcolor{dependence-purple}{\mathrm{Dep}\left(J,\bar{J}\right)}$, and thus $\mathrm{Dep} \left (J, \bar{J} \right ) := \Psi\left (J, \bar{J} \right ) - \mathrm{Int}\left (J, \bar{J} \right )$. We can estimate $\Psi(J, \bar{J})$ as 
\begin{align*}
    \widehat{\Psi}^{\mathcal{I}_{te}}_{\hat{f}} \left (J, \bar{J} \right ) = \hat{v}_{\hat{f}, \LL^2}^{\mathcal{I}_{te}} \left (J\cup \bar{J} \right ) - \hat{v}_{\hat{f}, \LL^2}^{\mathcal{I}_{te}} \left (J \right ) - \hat{v}_{\hat{f}, \LL^2}^{\mathcal{I}_{tr}} \left (\bar{J} \right )
\end{align*}
and get 
\begin{align*}
    \widehat{\mathrm{Dep}}_{\hat{f}, \hat{g}}^{\mathcal{I}_{te}} \left (J, \bar{J} \right ) := \widehat{\Psi}_{\hat{f}}^{\mathcal{I}_{te}} \left (J, \bar{J} \right ) - \widehat{\mathrm{Int}}_{\hat{f}, \hat{g}}^{\mathcal{I}_{te}} \left (J, \bar{J} \right ).
\end{align*}
We summarize the estimation procedure in Algorithm \ref{alg:estimation-dep-int}.
\begin{algorithm}[h]
\caption{Estimation of $\textcolor{interaction-green}{\mathrm{Int}}$ and $\textcolor{dependence-purple}{\mathrm{Dep}}$ Using Empirical Risk on Test Data}\label{alg:estimation-dep-int}
\begin{algorithmic}[1]
    \Input Feature indices $J$, data $(x, y)^{\mathcal{I}}$, split into train and test $\mathcal{I}^{tr} \cup \mathcal{I}^{te} = \mathcal{I}$, model $\hat{f}$.  
    \Output Estimates of \textcolor{interaction-green}{interaction surplus $\mathrm{Int} \left (J, \bar{J} \right )$} and \textcolor{dependence-purple}{main effect dependencies $\mathrm{Dep} \left ( J, \bar{J} \right )$}.
    \State $\hat{f}_J, \hat{f}_{\bar{J}}, \hat{f}_\emptyset \gets$ $\LL^2$ fits on $(x_S,y)^{\mathcal{I}_{tr}}$ for $S=J$, $S=\bar{J}$, and $S=\emptyset$.
    \State $\hat{g} \gets$ $\LL^2$ fit of GGAM in $J, \bar{J}$ on $(x,y)^{\mathcal{I}_{tr}}$
    \State $\textcolor{interaction-green}{\widehat{\mathrm{Int}}_{\hat{f}, \hat{g}}^{\mathcal{I}_{te}} \left (J, \bar{J} \right)}$ $\gets \hat{v}_{\hat{f}, \LL^2}^{\mathcal{I}_{te}} \left ( J, \bar{J} \right ) - \hat{v}_{\hat{g}, \LL^2}^{\mathcal{I}_{te}} \left ( J, \bar{J} \right )$
    \State $\widehat{\Psi}^{\mathcal{I}_{te}}_{\hat{f}} \left (J, \bar{J} \right ) \gets \hat{v}_{\hat{f}, \LL^2}^{\mathcal{I}_{te}} \left (J, \bar{J} \right ) - \hat{v}_{\hat{f}, \LL^2}^{\mathcal{I}_{te}} \left (J \right ) - \hat{v}_{\hat{f}, \LL^2}^{\mathcal{I}_{te}} \left (\bar{J} \right )$ \Comment{uses $\hat{f}_J, \hat{f}_{\bar{J}}, \hat{f}_{\emptyset}$}
    \State $\textcolor{dependence-purple}{\widehat{\mathrm{Dep}}^{\mathcal{I}_{te}}_{\hat{f}, \hat{g}}  \left (J, \bar{J} \right )} \gets \widehat{\Psi}^{\mathcal{I}_{te}}_{\hat{f}} \left (J, \bar{J} \right ) - \textcolor{interaction-green}{\widehat{\mathrm{Int}}_{\hat{f}, \hat{g}}^{\mathcal{I}_{te}} \left (J, \bar{J} \right)}$
\end{algorithmic}
\end{algorithm}
\paragraph{Estimation of Main Effect Cross-Predictability and Covariance}{
We recall that $\mathrm{Dep}\left (J, \bar{J} \right )$ is the sum of cross-predictability and main effect covariance
\begin{align*}
\textcolor{dependence-purple}{\mathrm{Dep}\left(J,\bar{J}\right)}  
        := \textcolor{cross-pred-purple}{
                   \underbrace{\Var\left (\E\left (g^*_J\mid X_{\bar{J}}\right )\right ) + \Var\left (\E\left (g^*_{\bar{J}}\mid X_J\right )\right) 
                   }_{\text{Cross-Predictability $\mathrm{CP}\left (J, \bar{J} \right)$}}}
        + \textcolor{cov-purple}{
                   \underbrace{2\Cov\left (g^*_J, g^*_{\bar{J}}\right )
                   }_{\text{Covariance $\mathrm{CO}\left (J, \bar{J} \right)$}}}.    
\end{align*}
Given an estimate of $\mathrm{Dep}\left ( J, \bar{J} \right )$, we can now either estimate the cross-predictabilty $\mathrm{CP}\left (J, \bar{J} \right)$ or the covariance $\mathrm{CO}\left (J, \bar{J} \right)$ and get the respective other term as the difference of both.\\
In our experiments, we estimate the covariance of the GGAM components, which is efficient to compute, since the GGAM components are readily available, and the covariance is a comparatively cheap computation (Algorithm \ref{alg:estimation-cross-pred-cov}). On the other hand, estimating the cross-predictability is also possible but requires two further fits for approximating $\E(\hat{g}_J\mid X_{\bar{J}})$ and $\E(\hat{g}_{\bar{J}}\mid X_J)$.
}

\begin{algorithm}[h]
\caption{Estimation of \textcolor{cross-pred-purple}{cross-predictability} and \textcolor{cov-purple}{covariance}}\label{alg:estimation-cross-pred-cov}
\begin{algorithmic}[1]
    \Input $\textcolor{dependence-purple}{\widehat{\text{Dep}}^{\mathcal{I}_{te}}_{\hat{f}, \hat{g}}}  \left (J, \bar{J} \right )$, GGAM $\hat{g} = \hat{g}_J + \hat{g}_{\bar{J}}$ fitted on $(x, y)^{\mathcal{I}_{tr}}$ to minimize $\LL^2$
    \Output Estimates of \textcolor{cross-pred-purple}{cross-predictability $\mathrm{CP}\left (J, \bar{J} \right )$} and \textcolor{cov-purple}{covariance $\mathrm{CO}\left (J, \bar{J} \right )$}
    \State $\textcolor{cov-purple}{\widehat{\mathrm{CO}}_{\hat{f}, \hat{g}}^{\mathcal{I}_{te}}} \left (J, \bar{J} \right ) \gets 2\Cov \left (\hat{g}_J\left(x^{\mathcal{I}_{te}}\right), \hat{g}_{\bar{J}}\left(x^{\mathcal{I}_{tr}}\right) \right)$
    \State $\textcolor{cross-pred-purple}{\widehat{\mathrm{CP}}_{\hat{f}, \hat{g}}^{\mathcal{I}_{te}}} \left (J, \bar{J} \right ) \gets \textcolor{dependence-purple}{\widehat{\text{Dep}}_{\hat{f}, \hat{g}}^{\mathcal{I}_{te}} \left (J, \bar{J} \right )} - \textcolor{cov-purple}{\widehat{\mathrm{CO}}_{\hat{f}, \hat{g}}^{\mathcal{I}_{te}} \left (J, \bar{J} \right )}$ 
\end{algorithmic}
\end{algorithm}

\subsection{Implementation and Code}
\label{appendix:estimation-implementation}

\paragraph{Python Package} 
{We implemented the method in a python package called \texttt{dipd}. The package and installation instructions are available via \url{https://github.com/gcskoenig/dipd}.\\
For the GGAMs, we rely on the \texttt{interpretML} package that implements so-called explainable boosting machines \citep{nori2019interpretml}. Furthermore, we use \texttt{numpy}, \texttt{pandas}, \texttt{matplotlib}, \texttt{seaborn}, \texttt{tqdm}, \texttt{scipy}, \texttt{statsmodels}, and \texttt{scikit-learn} in the most current version available in \texttt{python3.11.7} (a full list of the installed packages including version can be found in the linked repository).\\
In all experiments involving the explainable boosting machine, we fit the model using the default hyperparameters, except that we specify which interventions are and are not allowed. In cases where we use a linear model as GAM (or GGAM), we use the OLS implementation in the \texttt{statsmodels} package, also using default hyperparameters.}
\paragraph{Experiments}
{To reproduce our experiments, we refer to the instructions in our repository (\url{https://github.com/gcskoenig/dipd-experiments}). As follows, we summarize the key parameters for each of the experiments, that is, which models were used, how much data was used, and how we split test and training data.\\
For the student examples (Example \ref{Example: student redundancy} to \ref{example: binary interaction}), we sampled $10^5$ observations and randomly split the dataset into 80\% training and 20\% test data. For the first two examples, we use a linear model as GAM; for the interaction example, the explainable boosting machine.\\
For Example \ref{example:cooperative-forces-cancel-out}, where the cooperative forces cancel out, we sample $10^5$ data points, again hold out 20 \% of the data for testing, and leverage explainable boosting machines for all model fits.\\
For the real-world applications, we split the data into $10$ folds, train on the respective training, and compute the scores on the test data. For all model fits, we leverage the explainable boosting machines.\\
For the example introduced in Appendix \ref{app:examples:correlation vs main effect dependencies} we sampled $10^6$ observations and used explainable boosting machines for all model fits.\\
In Appendix \ref{appendix:more-experimental-results} we present additional experimental results. In all cases we use test train splits with $20$\% test data and rely on explainable boosting machines.
}
\paragraph{Compute}
{The experiments were run on a MacBook Pro with M3 Pro Chip. The illustrative examples all ran in less than one minute, the DIP decompositions of LOCO on the housing and wine dataset took about ten minutes each. The DIP decompositions of SAGE on the wine and housing datasets (Appendix \ref{appendix:additional-results}) took about one and a half hours each.}

\newpage
\section{EXPERIMENTAL RESULTS}
\label{appendix:more-experimental-results}

In the main paper, we decomposed the LOCO scores on the wine quality dataset using DIP (Section \ref{section: experiments}, Figure \ref{fig:experiments-results} left). In this section, we present additional experimental results. In Section \ref{appendix:detailed-results}, we investigate the pairwise relationships of variables using DIP. In Appendix \ref{appendix:additional-results}, we additionally present the results of a DIP decomposition of SAGE values. The code for all presented experiments is availble in the repository accompanying the paper (and \url{https://github.com/gcskoenig/dipd-experiments}).\\
For more detailed results regarding the experiments in the main paper (such as standard deviations or the results for individual folds) we refer to the repository accompanying the paper.\\

\subsection{Detailed Analysis on the Wine Quality Dataset}
\label{appendix:detailed-results}

\subsubsection{Pairwise Decompositions Reveal Cooperation Partners} \label{appendix:pairwise-wine-quality}

\begin{figure}[h]
    \centering
    \hfill
    \begin{subfigure}{0.49 \textwidth}
        \centering
        \includegraphics[width=0.99 \linewidth]{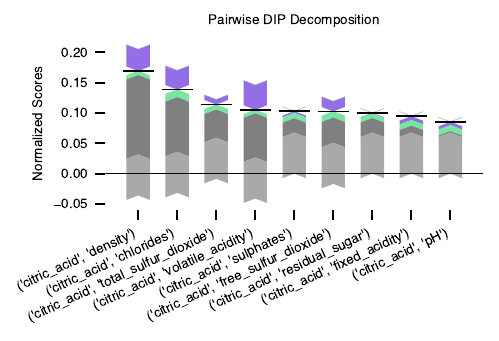}
        \caption{Citric Acidity}
        \label{fig:pairwise-wine-quality-citric-acidity}    
    \end{subfigure}
    \hfill
        \begin{subfigure}{0.49 \textwidth}
        \centering
        \includegraphics[width=0.99 \linewidth]{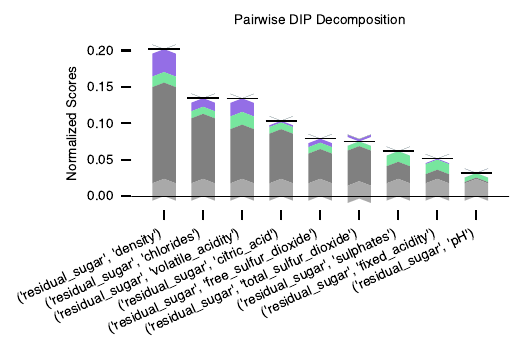}
        \caption{Residual Sugar}
        \label{fig:pairwise-wine-quality-sugar}    
    \end{subfigure} 
    \hfill
    \caption{Pairwise DIP decomposition on the Wine Quality Dataset. The two gray bars indicate the standalone contributions of each of the features in the pair, the black horizontal lines indicate the total performance than can be achieved with all features. The interaction surplus is highlighted green, the main effect dependencies purple.} \label{}
\end{figure}

In Section \ref{section: experiments}, we found that the variable \emph{citric}
 acidity has a large standalone contribution, but is largely redundant with the remaining features. Furthermore, we found that main effect dependencies between \emph{residual sugar} and the remaining variables have a positive impact on the performance. In this section, we leverage pairwise DIP decompositions to better understand their relationship with specific individual variables.\\
First, we have a closer look at the role of \textit{citric acidity}. We use DIP to understand which specific variables the feature shares information with (Figure \ref{fig:pairwise-wine-quality-citric-acidity}). The decomposition reveals that \textit{citric acidity} has a large negative contribution of main effect dependencies when combined with \textit{volatile acidity}, \textit{density}, and \textit{chlorides}, indicating that the variables have similar roles for the target.\\

Moreover, we are interested in the role of \textit{residual sugar} for the quality of a wine. We recall that the main effect dependencies between the variable and the remaining variables had a positive impact on the joint performance (and the LOCO score, Figure \ref{fig:experiments-results}).  %
Using pairwise DIP decompositions (Figure \ref{fig:pairwise-wine-quality-sugar}) we find a large positive effect of the main effect dependencies for the pairing with \textit{density}, \textit{chlorides}, and \textit{volatile acidity}. In other words, the DIP decompositions indicate that in these three pairings the variables have opposing relationships with target that are only revealed when analyzed jointly.
\footnote{We note that in the pairwise DIP, interactions play a more prominent role. This was not the case in the DIP decomposition of the LOCO scores. The reason is that while we previously analyzed the role of residual sugar when added to the remaining variables, we now analyze the role of the variable when added to just one other feature. In this bivariate setting, interactions are necessary to explain what previously could be explained with other variables.}
As follows, we focus on analyzing the relationship with \textit{density} using exploratory data analysis.\\

\subsubsection{Residual Sugar and Density Have Opposing Effects that Cancel Out}\label{appendix:relation-density-sugar-quality}

As follows, we have a closer look at the relationship between density, residual sugar, and wine quality. More specifically, we compare a pairwise histogram plot between a feature and target with the plot when additionally conditioning on the respective other feature.\\
While quality decreases with density (correlation coefficient $-0.30$, Figure \ref{fig:hist:density-quality}), no clear trend can be observed when looking at the pairwise relationship between sugar and quality (correlation coefficient $-0.03$, Figure \ref{fig:hist:sugar-quality}). When conditioning on sugar, the negative relationship between density and quality becomes more pronounced (mean correlation over all bins $-0.36$, Figure \ref{fig:hist:density-quality-by-sugar}). A positive relationship between sugar and quality becomes visible when conditioning on density (mean correlation over all bins $0.17$, Figure \ref{fig:hist:sugar-quality-by-density}). This suggests that sugar and density have opposing effects on the target but are positively correlated (coefficient $0.55$), such that their effects (partially) cancel each other out unless analyzed jointly.\\

\begin{figure}[h]
    \centering
    \begin{subfigure}{0.19\linewidth}
        \centering
        \includegraphics[width=0.99\linewidth]{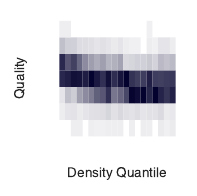}
        \caption{} \label{fig:hist:density-quality}
    \end{subfigure}
    \hfill
    \begin{subfigure}{0.79\linewidth}
        \centering
        \includegraphics[width=0.99\linewidth]{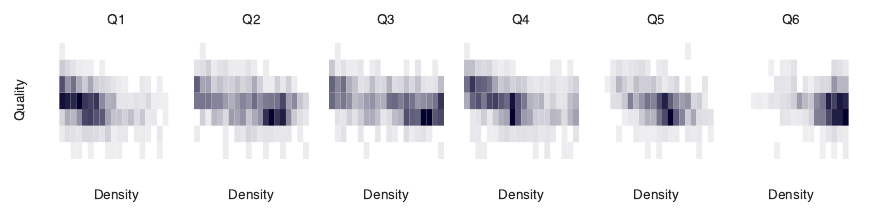}
        \caption{}\label{fig:hist:density-quality-by-sugar}
    \end{subfigure}
    \\
    \begin{subfigure}{0.19\linewidth}
        \centering
        \includegraphics[width=0.99\linewidth]{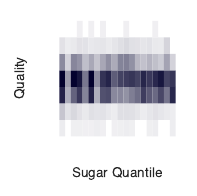}
        \caption{} \label{fig:hist:sugar-quality}
    \end{subfigure}
    \hfill
    \begin{subfigure}{0.79\linewidth}
        \centering
        \includegraphics[width=0.99\linewidth]{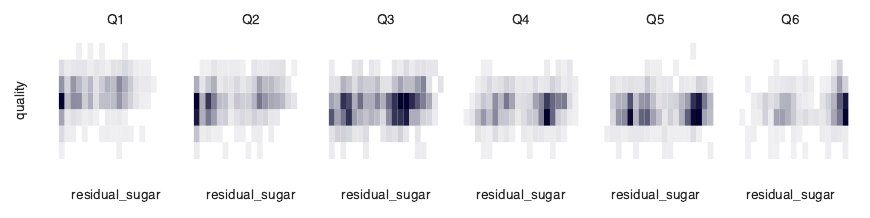}
        \caption{}\label{fig:hist:sugar-quality-by-density}
    \end{subfigure}
    \caption{Histogram plots showing the pairwise relationship of density and quality (top row), and sugar and quality (bottom row). To allow a visualization using a heatmap, we first encode the two features as ordinal variables with twenty levels. All levels have the same probability mass and correspond to equal-size quantile ranges. The left plots (Figure \ref{fig:hist:density-quality} and Figure \ref{fig:hist:sugar-quality}) show the pairwise relationship between each of the features and the target. The right plots (Figure \ref{fig:hist:density-quality-by-sugar} and \ref{fig:hist:sugar-quality-by-density}) show the relationship when conditioning on the respective partner. More specifically, to visualize the conditional density using a histogram plot, we discretize the conditioning variable into six equally sized bins (Q1 to Q6), and create one histogram plot for each bin.
    In both cases, the variable's relationship with the target becomes more visible when conditioning on the cooperation partner.
    }
    \label{fig:histograms}
\end{figure}

\subsection{DIP Decomposition of SAGE on the Wine Quality and California Housing Datasets}
\label{appendix:additional-results}

\begin{figure}[h]
    \centering
    \begin{subfigure}{0.49 \linewidth}
        \includegraphics[width=0.9\linewidth]{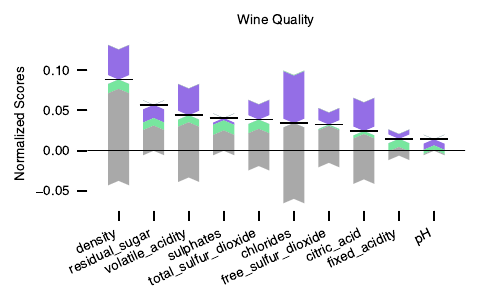}
        \caption{Wine Quality Dataset \citep{cortez2009modeling}.}
        \label{fig:sage-wine-quality}
    \end{subfigure}
    \hfill
    \begin{subfigure}{0.49 \linewidth}
        \includegraphics[width=0.9\linewidth]{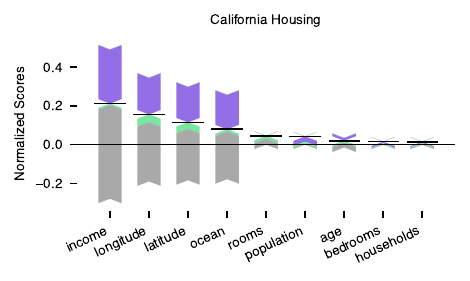}
        \caption{California Housing Dataset \citep{geron2022hands}.}
        \label{fig:sage-housing}    
    \end{subfigure}
    
    \caption{DIP Decompositions of SAGE values. For each surplus evaluation}\label{fig:sage-results}
\end{figure}

The DIP decomposition cannot only be applied to LOCO, but more generally to any explanation technique that is based on comparing the predictive power for different sets of features ($v(S \cup T) - v(S)$). One prominent example are so-called SAGE values.\\
As follows, we decompose SAGE values using DIP. For their computation we sampled $100$ orderings of the features, which each define which feature is added to which coalition of other features. Then, the SAGE value for a feature $j$ is the mean surplus $v(C \cup j) - v(C)$ achieved over the coalitions $C$ implied by the orderings. Each surplus $v(C \cup j) - v(C)$ can be decomposed into a standalone contribution and the cooperative impact. The final SAGE value is the standalone contribution plus the average cooperative impact (details in Appendix \ref{app: shapley effects}).\\
The results are presented in Figure \ref{fig:sage-results}. We apply DIP to the cooperative impact for each coalition, and present the respective average scores for \textcolor{interaction-green}{interaction surplus} and \textcolor{dependence-purple}{main effect dependencies}.\\
First, we have a look at the wine dataset. The feature \emph{citric acidity} again receives a relatively small SAGE score. The DIP decomposition reveals that the feature is relevant standalone, but receives a small score due to its redundancy with the remaining features. The feature \textit{residual sugar} has a high SAGE score. The DIP decomposition reveals that it has little standalone contribution, but contributes via interactions and positive main effect dependencies.\\
In the California housing dataset, we again observe that \emph{longitude} and \emph{latitude} are, to a large degree, important due to interactions. The feature \emph{ocean proximity} has a large standalone contribution, but receives a relatively small SAGE score due to its redundancy with the remaining features.

\vfill

\end{document}